\documentclass[10pt,onecolumn]{article}
\usepackage{subcaption}
\usepackage{listings}

\usepackage{ISG_style}
\usepackage{graphicx}
\usepackage{color}
\usepackage{amsmath, amsthm, amssymb, amsfonts, mathrsfs}
\usepackage{bbm}
\usepackage{dsfont}
\usepackage{flushend}

\usepackage{cite}
\usepackage{authblk}

\newtheorem{lemma}{Lemma}
\newtheorem{theorem}{Theorem}

\DeclareMathOperator*{\arginf}{arg\,inf}

\DeclareMathOperator*{\argmax}{arg\,max}
\DeclareMathAlphabet\mathbfcal{OMS}{cmsy}{b}{n}
\makeatletter
\newcommand{\xRightarrow}[2][]{\ext@arrow 0359\Rightarrowfill@{#1}{#2}}
\makeatother

\newcommand*\diff{\mathop{}\!\mathrm{d}}

\usepackage{algorithm}
\usepackage{algpseudocode}
\usepackage{mathtools}
\usepackage[font=small,justification=centering]{caption}

\title{\vspace{-.0 in}  Active Sampling of Multiple Sources for Sequential Estimation}
\date{}

\author{
Arpan Mukherjee$^*$  \qquad  Ali Tajer\thanks{A. Mukherjee and A. Tajer are  with the Electrical, Computer, and Systems Engineering Department, Rensselaer Polytechnic Institute, Troy, NY.}\qquad Pin-Yu Chen$^\dagger$ \qquad Payel Das
\thanks{P.-Y. Chen and P. Das are with the IBM Thomas J. Watson Research Center, Yorktown, NY.}
}

\begin{document}
	
	\maketitle
	\allowdisplaybreaks
	
	\begin{abstract}
		Consider $K$ processes, each generating a sequence of identical and independent random variables. The probability measures of these processes have random parameters that must be estimated. Specifically, they share a parameter $\theta$ common to all probability measures. Additionally, each process $i\in\{1, \dots, K\}$ has a private parameter $\alpha_i$. The objective is to design an active sampling algorithm for sequentially estimating these parameters in order to form reliable estimates for all shared and private parameters with the fewest number of samples. This sampling algorithm has three key components: (i)~data-driven sampling decisions, which dynamically over time specifies which of the $K$ processes should be selected for sampling; (ii)~stopping time for the process, which specifies when the accumulated data is sufficient to form reliable estimates and terminate the sampling process; and (iii)~estimators for all shared and private parameters. Owing to the sequential estimation being known to be analytically intractable, this paper adopts  \emph {conditional} estimation cost functions, leading to a sequential estimation approach that was recently shown to render tractable analysis. Asymptotically optimal decision rules (sampling, stopping, and estimation) are delineated, and numerical experiments are provided to compare the efficacy and quality of the proposed procedure with those of the relevant approaches.
	\end{abstract}
	
	\section{Introduction}
	\subsection{Overview}
	Consider the canonical estimation problem, in which we have a collection of probability measures $\mcP\triangleq\{\P(\cdot\med \theta)\; : \; \theta\in\Theta\}$ defined over a common measurable space. The nature picks $\theta$, the statistician draws samples from $\P(\cdot\med\theta)$, and the objective is to use these samples to estimate $\theta$. Building up on this canonical model, assume that we have a collection of $K$ probability measures $\mcP_i\triangleq \{\P_i(\cdot\med\theta)\; :\; \theta\in\Theta\}$ for $i\in[K]\triangleq\{1,\cdots,K\}$. Similarly, the nature selects $\theta$ and the statistician is given the freedom to collect samples from any one of the $K$ models, and the objective is to form a reliable estimate for $\theta$. With the objective of estimating $\theta$ with the fewest number of samples, a fundamental question pertains to which model(s) are the most reliable for estimating $\theta$. If we can determine, {\em a priori}, which model is expected to be the most informative about $\theta$ in its entire range, then the answer is clear: always sample from that model. For instance, when $K=2$ and   
	\begin{align}
		\P_1(\cdot \med \theta) \sim \mcN(\theta,1) \quad\mbox{and} \quad 
		\P_2(\cdot \med \theta) \sim \mcN(\theta,2)\ ,
	\end{align}   under the mean-squared error cost function, $\P_1$ is always a more reliable model to use. In general, however, different models can be selectively more descriptive about $\theta$ in different regimes. For instance, for $\Theta\in[0,1]$ consider the following two models
	\begin{align}\label{eq:model1}
		\P_1(\cdot \med \theta) \sim \mcN(0,\theta) \quad\mbox{and} \quad 
		\P_2(\cdot \med \theta) \sim \mcN(0,1-\theta)\ .
	\end{align} 
	It can be readily verified that under the mean-squared cost function, when $\theta\in(0,1/2)$, $\P_1$ is the more informative model, and when $\theta\in(1/2,1)$, $\P_2$ is the more informative one. Thus, without knowing the value of $\theta$, it is impossible to specify which model is more informative. In such scenarios, any sampling action that focuses on only one model is sub-optimal. An optimal strategy will involve alternating between models until one can be identified as most informative with sufficient confidence. The complexity of such decisions can be further compounded when there are more parameters involved. As an example, consider a generalization of (\ref{eq:model1}) in which the mean values are also unknown, i.e., 
	\begin{align}
		\P_1(\cdot \med \theta,\mu_1) \sim \mcN(\mu_1,\theta) \;\; \mbox{and} \;\; 
		\P_2(\cdot \med \theta,\mu_2) \sim \mcN(\mu_2,1-\theta)\ ,
	\end{align} 
	where $\theta\subseteq[0,1]$ and $\mu_1,\mu_2\in\R$. In these models, even when the most informative model for $\theta$ is known, we still need to draw samples from both models in order to be able to estimate $\mu_1$ and $\mu_2$, since $\mu_i$ can be estimated exclusively from model $i\in\{1,2\}$. In this paper, we consider a general setting of $K$ models that can capture both aspects discussed (i.e., shared and private parameters). The objective is to design a sequential and data-adaptive sampling procedure such that we can form sufficiently reliable estimates for the parameters involved with the fewest number of samples. 
	\par The design of such a sequential procedure involves performing three intertwined tasks dynamically over time. The first decision pertains to forming and updating estimates about the parameters over time. Besides the ultimate interest in the estimates, they also guide the sampling process. The second decision specifies the next model to be selected for sampling. These decisions at time $t$ are formed based on the sequence of models selected up to time $t$, the data collected, and the estimates formed. Finally, a stopping decision has to be specified, which terminates the sampling process when sufficient confidence about the fidelity of the estimates has been reached. Designing such data-acquisition and inference mechanisms is related to active (controlled) sampling for sequential design of experiments, the foundations of which was laid out by Chernoff for binary composite hypothesis testing \cite{Chernoff} through incorporating a data-acquisition process that dynamically decides about taking a finite number of data-acquisition actions. Under the assumption of uniformly distinguishable hypothesis and statistically independent control actions, there exists a rich body of literature on greedy algorithms (making decisions with best immediate return) that achieve optimality in the asymptote of diminishing rates of erroneous decisions. 
	
	Unlike the extensive literature for detection and classification problems, active sampling  is far less investigated for estimation problems. Sequential designs of experiments for detection and estimation have a wide range of real-world applications. In machine learning, we are generally provided with a large dataset, from which machine learning literature almost universally assumes selecting a subset of samples uniformly at random, with the objective of performing various learning tasks. However, a network-guided sequential sampling strategy is shown to provide significant improvements in terms of sample complexity in~\cite{TajerMarkov,PAC,tong:01}. This framework also finds strong resemblance to that of multi-armed bandits~\cite{bandit}, which have a broad set of applications such as clinical drug testing and trial, brain and behavior modelling and dynamic pricing ~\cite{app4}, \textcolor{black}{computerized adaptive testing (CAT)~\cite{rev1}}, detecting intrusions in computer networks, ~\cite{app5} and target detection in radar systems~\cite{app6}. Next, we provide an overview of the related literature on active sampling for inference problems that are most relevant to the scope of this paper.

	\subsection{Related Literature}
\noindent	\textbf{Active Sampling for Detection.} The notion of active sequential detection was first studied for binary composite hypothesis testing in \cite{Chernoff}, providing asymptotically optimal design rules for sequential hypothesis testing under uniformly distinguishable hypotheses. This approach and its extensions to broader settings in~\cite{TajerMarkov} and~\cite{Bessler,Albert1961,box1967,Meeter,Atia2, Veeravalli2,Zhao:IT15,vaidhiyan2015active, Naghshvar2013,Wang:Arxiv15,Zhao:SP14,Zhao:SP15,TajerDetectionLearning,J26,Yahav}, sequentially take control actions that greedily maximize the immediate return. Specifically, the rules identify the most likely decision at every step, and selects the action that maximally reinforces this decision. The studies in \cite{Bessler} and \cite{Albert1961} extend the results to the cases of an infinite number of available actions, and an infinite number of true states of nature, respectively. The studies in \cite{box1967} and \cite{blot73} provide alternate sampling strategies that are empirically shown to outperform Chernoff's rule in the non-asymptotic regime. In a more recent study, the assumptions of having uniformly distinguishable hypothesis is relaxed in \cite{Atia2}, and a modified Chernoff's rule is designed for multi-hypothesis testing. This approach introduces a sequence of intervals during which actions are selected uniformly at random instead of selecting the action that maximizes the immediate return. An extension to a stationery Markov model is investigated in \cite{Veeravalli2}. Detection of an anomalous process from a finite set of anomalous processes has been considered in \cite{Zhao}, where it is shown that Chernoff's rule is asymptotically optimal even without the assumption of uniformly distinguishable hypotheses, or using uniform selections at certain instants as prescribed in \cite{Atia2}. 
	More applications of Chernoff's rule can be found in \cite{vaidhiyan2015active,mitra} and~\cite{atiaClassification}.
	Active sequential hypothesis testing coupled with a switching cost for switching between actions has been studied in \cite{vaidhiyan2015active}. In this study, a modification of the Chernoff's procedure is shown to be asymptotically optimal in the regime of diminishing error probabilities. The study in ~\cite{mitra} considers the problem of active state tracking for a dynamic stochastic system whose states are varying through a discrete time Markov chain, and a dynamic programming-based approach is proposed for estimating the underlying states. The problem of active classification of graphs based on connectivity through partially observable nodes is studied in ~\cite{atiaClassification}. The problem of active model selection in Markov random fields (MRFs) is studied in \cite{TajerMarkov}, where data acquisition is modelled as a sequential process. Induced by the underlying dependence structure of an MRF, different sampling decisions over time are not necessarily statistically independent, rendering Chernoff's rule sub-optimal in this setting, even in asymptotic regimes.  This arises mainly due to the fact that Chernoff’s rule focuses on maximizing the immediate return, without paying attention to the expected future return. The approaches in \cite{TajerMarkov} incorporates the effect of sampling actions on the expected rewards from future sampling decisions. Chernoff's test has been extended to a distributed setup in \cite{rangi2018distributed}, designing distributed and consensus-based Chernoff tests for decentralized detection.
	
	Other related literature on the sequential design of experiments for inference include the studies presented in \cite{schwarz1962} and~\cite{kiefer1963}, which propose approaches that take an initial number of samples according to certain predesignated rule in order to guess the true state of nature, and then select actions that maximize proper information measures based on the predicted states. Furthermore, \cite{lalley1986} studies the problem of sequential design of experiments in the finite-sample regime, and proposes a switching policy which achieves an average sample complexity that is within a constant factor of the optimal average sampling complexity, while being computationally simpler than the optimal solution. The problem of multi-hypothesis testing using multiple controls is studied in \cite{Naghshvar2013}, which frames the problem as a dynamic program, whose optimal solution is intractable. Thus, two heuristic policies are proposed and analyzed in both asymptotic and non-asymptotic regimes.
	
{\noindent \textbf{Active Sampling for Estimation -- Theoreitial Aspects}.} Compared to the extensive literature on active sampling for detection and classification problems, {the theoretical aspects of} parameter estimation in an active sampling setting are far less investigated. Sequential estimation, without the additional complexity of control actions, has a rich body of literature in statistics. Reviews of sequential point and interval estimation can be found in~\cite{rev3,rev4,ghosh2011sequential,seq_est1,seq_est2}. Furthermore, a representative list of active sampling for parameter estimation includes~\cite{rev1,Yahav}, and~\cite{yohai1973,rev2,Atia,MoustakidesISIT}.

	Sequential estimation from one time series is investigated in \cite{Yahav}, where a fixed sampling cost is charged every time a new observation is procured, and a unified cost function aggregating the estimation accuracy and the total cost of sampling the observations is adopted. The stochastic unified cost function captures a trade-off between estimation fidelity and agility of arriving at an estimate. This study formalizes the notion of asymptotic point-wise optimality (A.P.O.) for assessing the performance of any sequential design, and proposes sequential procedures that are asymptotically optimal.
	
	This study was extended in ~\cite{yohai1973} to active sampling for estimation by incorporating a set of control actions for collecting the data. In ~\cite{yohai1973}, the objective is to estimate a function of an unknown parameter underlying each control action, such that the parameters under different control actions are different. Hence, the data collected based on a control action does not provide any information about the parameters associated with other actions. In \cite{Atia}, this drawback is mitigated by considering the case of estimating a single shared unknown parameter, while allowing the distributions under different control actions to be non-identical. This study also considers estimating a linear function of a set of shared unknown parameters. In a different setting, an intertwined problem of model detection and model probability kernel estimation is considered in \cite{TajerDetectionLearning}, where in addition to discerning the true model (detection), an unknown probability kernel should be estimated too (estimation). Thus, the generalized stopping rule has to cater to both the quality of estimation and the intertwined task of detection. In all these studies, a unified cost function involving a linear combination of the estimation cost and the sampling cost are considered. Such approaches, in general, are known to be analytically intractable. To address this issue, a different approach to sequential estimation is introduced in \cite{MoustakidesISIT}, where a single control is assumed, and the conventional average estimation cost function is replaced by an average {\em conditional} estimation cost function. In~\cite{MoustakidesISIT}, the objective is to form a reliable estimate of an unknown parameter using the fewest number of samples such that a Bayes posterior risk falls below a pre-specified threshold. This study provides an exactly optimal solution for this setting. In this paper, we adopt the approach of~\cite{MoustakidesISIT} for the estimation cost function in the active sampling framework. In this framework, we propose asymptotically optimal sequential procedures, with the main emphasis being on the design of efficient active sampling strategies.
	
{	\noindent \textbf{Active Sampling for Estimation -- Applied Aspects}.	On the application side, one domain in which active sampling for estimation is used is in the context of computerized adaptive testing (CAT). The goal in this setting is to devise a sequence of sampling decisions (items) that match the trait level of an examinee, characterized by an unknown parameter. This general setting and objective of CAT, while relevant, has major differences in its settings, assumptions, objective, and performance metrics. First, CAT falls in the category of finite-horizon settings. Since it has a fixed number of items to choose from~\cite{rev1,CAT_nstop1,CAT_nstop2,CAT_nstop3} and each item can be used at most once. This enforces a hard constraint on its stopping time. Even though some investigations introduce an adaptive stopping time~\cite{CAT_stop1,CAT_stop2,CAT_stop3,wang2013deriving}, the stopping time may not exceed the number of items in the pool. In contrast, we consider an infinite-horizon setting, in which each experiment can be sampled as many times as necessary. In sequential analyses, finite versus infinite-horizon have distinct guarantees and designs. Secondly, CAT considers binary observations from the examinees, using the item response theory with a parameterized logistic regression model. This is distinct from the setting in this paper, which is much more general; essentially, we consider any model that satisfies some mild regularity conditions. Third, the objective function optimized in CAT is different from the one investigated in this paper. Specifically, CAT optimizes measures associated with the sample covariance, such as the largest eigenvalue or the trace~\cite{wang2013deriving}. On the contrary, we optimize the average sample complexity with a constraint on the estimation cost. This cost function has no precedence in the active sequential literature and has only been considered in~\cite{MoustakidesISIT} for the case of single control action. Finally, the key performance metrics investigated in CAT are the properties of the estimator, such as the bias and mean squared error reduction. In sequential statistics, the canonical performance metric is the average sample complexity~\cite{Chernoff,Veeravalli2,Veeravalli,TTEI,pmlr-v49-garivier16a,pmlr-v108-shang20a}.  The results investigated in this paper are based on characterizing the average sample complexity for the proposed procedures and analyzing their optimality.}
	
	\subsection{{Contributions}}
	
	{Our framework for active sampling for sequential parameter estimation consists of three intertwined decision rules: (i) a sequence of estimators for dynamically updating the estimates of interest over time, (ii) a sampling rule that specifies how the sampling decisions are guided by and adapted to the data, and (iii) a stopping rule, at which the sampling process is terminated, and the available data is deemed sufficient for forming sufficiently reliable estimates. The overall theoretical contribution is designing a combination of rules that collectively admit a form of optimality, addressing a setting for which theoretical guarantees are hitherto unknown. Besides the theoretical guarantee, we also specify our decision rules' relevance (similarities and differences) to the existing ones in active sampling and sequential estimation literature.
\begin{itemize}
	\item \textbf{Estimator.}  Generally, sequential estimation problems are not as analytically tractable as their detection counterparts, often rendering a lack of optimality guarantees. In this paper, we adopt a new estimation cost function introduced in~\cite{MoustakidesISIT}. This cost function focuses on the \emph{conditional} estimation cost, distinct from the conventional average estimation cost functions adopted in sequential estimation. This choice of the cost function facilitates tractable analysis and closed-form characterization of the decision rules that enjoy optimality guarantees. This cost function has no precedence in the active sampling frameworks. Besides the literature on the theoretical aspects, this cost function does not have precedence in the relevant applied domains (e.g., CAT), which generally focus on the conventional average estimation cost function, distinct from the average posterior cost function adopted in this paper.
	\item \textbf{Sampling rule.} Our setting and analysis generalize those of~\cite{MoustakidesISIT}, which focuses on a single experiment. We consider a two-fold generalization. First, we consider multiple experiments, where one experiment can be chosen at each instant. Secondly, each experiment can have both shared and private parameters, each having different guarantees on the estimation cost. Thus, the sampling rule, at its core, has to balance between greedily sampling the most informative experiment and ensuring that the guarantees for the less informative experiments are satisfied. We adopt a sampling rule that ensures an optimal trade-off between exploiting the most informative experiment while providing sufficient exploration for the other experiments. This is a key distinction from the existing theoretical investigation of active sampling for sequential estimation~\cite{Atia} and their applications, e.g., in CAT. In CAT, specifically, the active selection of the experiments has two main differences. First, they do not have the notion of shared versus private parameters and consequently do not have the inherent estimation trade-off involved. Secondly, CAT by design is a finite-horizon problem, rendering the design of the sampling rules fundamentally a dynamic programming problem. This is distinct from our setting, which is infinite-horizon with a stochastic stopping time. 
	\item\textbf{Stopping rule.} Finally, we propose a stopping rule that ensures that the respective guarantees on the estimation cost for both the shared and private parameters are satisfied while, in parallel, the average delay is minimized. Specifically, we provide universal lower bounds on the average delay required to achieve the guarantees on the estimation cost and then show that our proposed procedures match these bounds asymptotically. Furthermore, we show that the most relevant sampling rule in~\cite{Atia} is sub-optimal in the general setting. These stopping rules, being infinite-horizon, are quite distinct from those generally adopted in CAT.
	\end{itemize}}

	\section{Active Sampling Model}\label{Sec:single_param}

	Consider a finite set of experiments $\{S_1,\cdots,S_K\}$. When experiment $S_i$ for $i\in[K]\triangleq\lbrace 1,\cdots, K\rbrace$ is selected, it generates a random variable whose probability distribution depends on two unknown parameters $\theta\in\Theta\subseteq \R$ and $\alpha_i\in\mcA_i\subseteq\R$. We call $\theta$ the {\em shared} unknown parameter of all experiments, and call $\alpha_i$ the {\em private} unknown parameter of experiment $i\in[K]$. As an example, consider a sensor network in which each sensor produces noisy observations of a shared parameter $\theta$ (such as the temperature and pressure), and the task is to estimate the shared parameter from the observations obtained from the sensors. Furthermore, the stochastic processes corresponding to each sensor $i\in[K]$ could also have unknown private parameters $\alpha_i$ governing the data generation, such as the noise variance of each sensor. Our objective is to have a sequential design of experiments such that we can estimate $\theta$ and $\balpha \triangleq (\alpha_1,\cdots,\alpha_K)$ with the fewest number of experiments (samples). We denote the probability density function (pdf) of the samples generated by experiment $S_i$ by $f_i(\cdot\med \theta,\alpha_i)$, where for simplicity of notations we are assuming that the pdfs are well-defined without any zero over lower-dimensional manifolds. We also assume that the samples generated by an experiment are independent and identically distributed (i.i.d.). Our focus is on the Bayesian setting and we denote the prior pdfs of $\theta$ and $\alpha_i$ by $\pi_\theta$ and $\pi_i$, respectively. 

	\begin{figure*}[h]
\centering
				\includegraphics[width=0.8\linewidth]{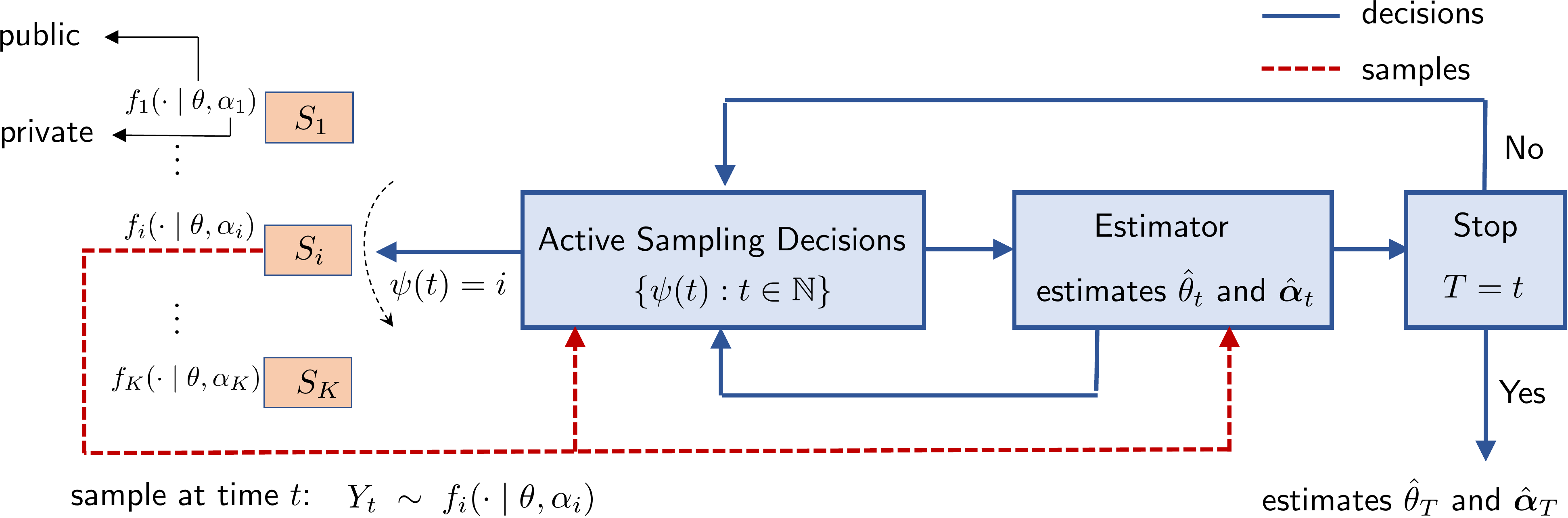} 
		\caption{Active sampling for sequential parameter estimation. $\hat\theta_t$ and $\hat\balpha_t$ denote the estimates of the shared and private parameters respectively, $\{\psi(t):t\in\N\}$ represents the sequence of selected experiments, and $T$ represents the stopping time.}
		\label{fig:full_unknown} 
\end{figure*}

	\subsection{Sampling Model}\label{Sec:Sampling_model_1}
	We consider a fully sequential data-acquisition mechanism, according to which we are allowed to gather one sample at-a-time. The sampling model is described in Figure~\ref{fig:full_unknown}.  Our primary goal is to identify a sequence of experiments for forming reliable estimates for $\theta$ and $\balpha$ using the minimum number of experiments on average. An optimal sampling process will require balancing a number of competing measures in an optimal way. The trade-offs we are facing include:
	\begin{itemize}
		\item {\bf Delay versus fidelity:} Collecting more samples, when used judiciously, results in improved estimation fidelity at the expense of increased delay in forming estimates.
		\item {\bf Model discrepancies:} Not all experiments are equally informative about the shared parameter. Nevertheless, {\em a priori}, it might not be known which experiments are the most informative about the realization of $\theta$. Hence, this creates a natural exploration versus exploitation trade-off, according to which the experiments serve a two-fold purpose: identifying the most informative experiments and using the data from these experiments to estimate~$\theta$.
		\item {\bf Shared versus private parameters:} While for estimating the shared parameter there is the tendency to identify the most informative experiment for $\theta$, for estimating the private parameters each experiment should be selected sufficiently often to render a reliable estimate for its private parameter.
	\end{itemize}
	To formalize the sampling process, we need to specify three decision-making tasks. First, we specify a sampling rule $\psi: \N\rightarrow [K]$ where $\psi(t)\in[K]$ denotes the experiment to be selected at time $t$. We denote the sample collected from the experiment $S_{\psi(t)}$ by $Y_t$. Accordingly, we define the {\em ordered} sets
	\begin{align}
		Y^t\;\triangleq\;\big\{ Y_1,\cdots,Y_t\big \}\ , \;\;\mbox{and} \;\; \psi^t\;\triangleq\; \big\{\psi(1),\cdots,\psi(t)\big\}\; .
	\end{align} 
	The observed filtration induced by a sampling rule $\psi$ is denoted by $\mcF_t^\psi\triangleq \sigma (Y^t,\psi^t)$\footnote{Unless otherwise stated, we use the short-hand $\F_t$ for $\F_t^\psi$, where the sampling rule $\psi$ is clear from the context.}. We assume that $\psi(t+1)$ is $\mcF_t^\psi$-measurable. The second decision-making task is a stopping rule that determines the instance at which we have accumulated sufficient evidence to form reliable estimates. Specifically, given a sampling rule $\psi$, we define an $\mcF_T^\psi$-stopping time that is a randomization $T$ such that $\{T=t\}\in\mcF_t^\psi$,
	for all $t\in\N$. Finally, corresponding to a given sampling rule $\psi$, we define $\hat\theta (Y^t,\psi^t)$ and $\hat\alpha_{i}(Y^t,\psi^t)$ as the estimators for $\theta$ and $\alpha_i$ respectively\footnote{Throughout the paper, sometimes we use the short-hands $\hat\theta_t$ and $\hat\alpha_{i,t}$ for $\hat\theta(Y^t,\psi^t)$ and $\hat\alpha_{i,t}(Y^t,\psi^t)$ respectively.}. We refer to the tuple 
	\begin{align}
		\Delta\;\triangleq\; (T,\psi,\Phi_T)\ ,
	\end{align}
	as the collection of rules involved in active sampling for sequential estimation, where we have defined 
	\begin{align}
		\Phi_t\;\triangleq\; (\hat\theta(Y^t,\psi^t), \hat\balpha(Y^t,\psi^t))\;.
	\end{align}
	
	\section{Problem Statements}
	In this section, we formalize the problem of active sampling for sequential estimation. There are two key figures of merit involved in characterizing the performance of the active sampling framework: the average delay (sample complexity) and the estimation costs incurred by the final estimates. There exists a tension between these two quantities, since improving one penalizes the other one. Specifically, improving the estimates necessitates collecting more samples, which in turn, penalizes the sample complexity. Capturing this trade-off, our formulation aims to maintain the estimation cost below a target threshold and, in parallel, minimize the average sample complexity. 
	\subsection{Estimating Shared Parameters}
	We start by focusing on the setting in which we have unknown shared parameters, and the private parameters are fully known. In this setting, we denote the pdf of the samples generated by the experiment $S_i$ by $f_i(\cdot\med\theta)$, for $i\in[K]$. Define $\ell(\theta,\hat\theta)$ as a non-negative cost function that captures the fidelity of the estimate $\hat\theta$ with respect to the ground truth $\theta$. Accordingly, $\E[\ell(\hat{\theta},\theta)]$ is the associated average cost of the estimate $\hat{\theta}$, where the expectation is taken with respect to the data and the prior distribution of $\theta$. 

	The posterior distribution of $\theta$ at time $t$ given the set of samples and control actions taken up to time $t$ is given by
	\begin{align}
		\pi^t_{\theta}(\theta) \;= \;\frac{\pi_\theta(\theta)\displaystyle \prod\limits_{i\in\psi^t} f_{i}(Y_i\;|\;\theta)}{\displaystyle \int_{v\in\Theta} \pi_\theta(v)\displaystyle \prod\limits_{i\in\psi^t} f_{i}(Y_i\;|\;v)\; \diff v}\; .
	\end{align}
	Based on the above posterior, we define the \textit{ conditional average} posterior cost for an estimate $\hat{\theta}_t$ formed at time $t$ as
	\begin{align}\label{eq:con1}
		{\sf C}(\hat\theta_t\;|\;\F_t)\;\triangleq\; \E_t\big[\ell(\hat\theta_t,\theta)\;|\;\F_t\big]\ ,
	\end{align}
	where $\E_t$ denotes expectations with respect to $\pi^t_\theta$. We denote the associated average posterior cost by 
	\begin{align}
		\label{eq:cost_C}
		{\sf C}(\hat\theta_t)\;=\;\E_t[\ell(\theta,\hat\theta_t)]\ .
	\end{align}
	\textcolor{black}{Note that the estimation cost ${\sf C}(\hat\theta_t)$ depends on the choice of the cost function $\ell(\cdot,\cdot)$. Thus, the algorithm design and its optimality properties depend on the choice of the cost function. In this paper, we use the quadratic cost function for our analyses.} 
	The sequence of conditional posterior cost functions $\{{\sf C}(\hat\theta_t\med\F_t) : t\in\N\}$ and posterior cost functions $\{{\sf C}(\hat\theta_t) : t\in\N\}$ are $\mcF_t$-measurable and form increasing sequences of $\sigma$-fields. A natural approach to formulating sequential estimation is finding a solution~to:
	\begin{align}\label{eq:unsovable}
		\tilde{\mathcal{P}}(\tilde{\beta})\;\triangleq\;\begin{cases}
			\inf\limits_{\Delta}\;&\E[T]\\
			\text{s.t.}\; &{\sf C}(\hat\theta_t)\;\leq\;\tilde{\beta} 
		\end{cases}\ ,
	\end{align} 
	where $\tilde\beta\in\R_+$ controls the estimation quality. However, as discussed in~\cite{MoustakidesISIT} and its references [3-6], solving (\ref{eq:unsovable}) even in simpler settings, e.g., $K=1$, in which we do not have the action sampling decisions, is analytically intractable. In this paper, instead of (\ref{eq:unsovable}), we adopt the approach of~\cite{MoustakidesISIT} and in (\ref{eq:unsovable}) replace the average posterior cost ${\sf C}(\hat\theta_t)$ with the conditional average posterior cost ${\sf C}(\hat\theta_t\med\F_t)$. Hence our objective is to minimize the average sample size such that the conditional average posterior cost falls below a prescribed threshold, formalized as  
	
	\begin{align}\label{eq:var}
		\mcP(\beta)\triangleq \left\{
		\begin{array}{ll}
			\inf\limits_\Delta & \E[T]
			\\
			\mbox{s.t.} & {\sf C}(\hat\theta_T\;|\;\F_T) \leq \beta 
		\end{array}\right. \ ,
	\end{align}
	where $\beta\in\R_+$ controls the quality of the estimate $\hat\theta_T$ given $\F_T$. 
	
	\subsection{Estimating Shared and Private Parameters}
	Next, we consider a generalization of the framework, such that in addition to the shared parameter, experiment $S_i$ also has a private statistically independent parameter $\alpha_i$ to estimate. To formalize this, we denote the posterior joint pdf of parameters $\theta$ and $\{\alpha_i : i\in[K]\}$ conditioned on $Y^t$ and $\psi^t$ by $g_t$. By defining the space $\mcA\triangleq \prod\limits_{i=1}^K \mcA_i$ we have 
	\begin{align}
		g_t(\theta,\balpha)\;=\;\;\frac{\pi_\theta(\theta)\displaystyle\prod\limits_{j=1}^K\pi_i(\alpha_j)\displaystyle \prod\limits_{i=1}^t f_{i}(Y_i\;|\;\theta,\balpha)}{\displaystyle \int_{\mathbf{u}\in\mcA}\int_{v\in\Theta} \pi_\theta(v)\displaystyle\prod\limits_{j=1}^K\pi_j(u_j)\displaystyle \prod\limits_{i=1}^t f_{i}(Y_i\;|\;v,u)\; \diff v \diff \mathbf{u}}\; .
	\end{align}
	Subsequently, the marginal posterior pdf of $\theta$ conditioned on $Y^t$ and $\psi^t$ is denoted by
	\begin{align}
		\pi_\theta^t(\theta)\triangleq&\int\limits_{\mathbf{v}\in\mcA} g_t(\theta,\bv)\diff \mathbf{v}\;.
	\end{align}
	Furthermore, let $S_i^t \triangleq \{s\in\{1,\cdots,t\}\; :\; \psi(s)=i\}$ denote the set of instances up to time $t$ at which experiment $i\in[K]$ is selected. Hence, the posterior pdf of $\alpha_i$ conditioned on $\theta$, $Y^t$ and $\psi^t$ is
	\begin{align}
		h_i^t(\alpha_i)\;=\; \frac{\pi_i(\alpha_i)\displaystyle\prod\limits_{s\in S_i^t}f_i(Y_s\med \alpha_i,\theta)}{\displaystyle\int_{v\in\mcA_i}\pi_i(v)\displaystyle\prod\limits_{s\in S_i^t}f_i(Y_s\med v,\theta) \diff v}\ .
	\end{align}
	Based on the above definitions, we now define the \textit{conditional average estimation cost} for estimating $\alpha_i$ as 
	\begin{align}
		\label{eq:cost_D}
		{\sf D}(\hat\alpha_{i,t}\;|\;\F_t,\hat{\theta}_t)\;\triangleq\; \E_t^i\big[\ell(\hat\alpha_{i,t},\alpha_i)\;|\;\F_t,\hat\theta_t\big]\ ,
	\end{align}
	where $\E_t^i$ is the expectation with respect to $h_i^t(\alpha_i)$. This setting emphasizes a hierarchy of inference objectives in which the primary objective is estimating the shared parameter $\theta$. The estimate of $\theta$, subsequently, guides estimating the private parameters. By incorporating constants that capture the fidelity of the estimates for $\{\alpha_i : i\in[K]\}$, an optimal sequential estimation procedure can be found as the solution to:
	\begin{align}\label{eq:var_multi}
		\mcP(\bbeta)\triangleq \left\{
		\begin{array}{ll}
			\min\limits_\Delta & \E[T]
			\\
			\mbox{s.t.} & {\sf C}(\hat\theta_T\;|\;\F_T) \leq \beta \\
			\text{and} & {\sf D}(\hat\alpha_{i,T}\;|\;\F_T,\;\hat\theta_T) \leq \beta_i,\;\forall i\in[K]
		\end{array}\right. \ ,
	\end{align}
	where we have defined $\bbeta \triangleq [\beta,\beta_1,\cdots,\beta_K]$, and $\beta_i\in\R_+$ controls the estimation quality of $\alpha_i$.
	
	\subsection{Technical Assumptions}
	In this section, we provide the assumptions under which the performance guarantees are established. The assumptions are mainly necessary for the existence and consistency of the maximum likelihood (ML) estimates and the existence of the Fisher information (FI) measures of the relevant parameters. To proceed, corresponding to the pdfs $f_i(x\med\theta)$ (only shared parameters) and $f_i(x\med\theta,\alpha_i)$ (shared and private parameters), for $i\in[K]$, we define the log-likelihood functions
	\begin{align}
		&\lambda_i(x\med\theta)\;\triangleq\; \log f_i(x\med\theta)\ ,\quad\text{for }i\in[K]\ ,\\
		&\lambda_i(x\med\theta,\alpha_i)\;\triangleq\; \log f_i(x\med\theta,\alpha_i)\ ,\quad\text{for }i\in[K]\ .
	\end{align} 
	\begin{itemize}
		\label{assumption:1}
		\item[A$_1$:] Parameter spaces $\Theta$ and $\{\mcA_i : i\in[K]\}$ are assumed to be non-empty and compact.
		\item[A$_2$:]
		We assume that the ML estimate of the parameter $\theta$ exists, and it is finite. Specifically, when we have only the shared parameter, $\mathbb{E}[\lvert \lambda_i(x\;|\;\theta)\rvert]<+\infty$ for all $\theta\in\Theta$. Similarly, for the setting with shared and private parameters, for each $i\in[K]$, we assume $\mathbb{E}[\lvert \lambda_i(x\;|\;\theta,\alpha_i)\rvert]<+\infty$ for all $\theta\in\Theta$ and $\alpha_i\in\mcA_i$.
		\item[A$_3$:]
		The log-likelihood functions $\{\lambda_i(x\;|\;\theta) : i\in[K]\}$ are assumed to be continuous and differentiable for all $\theta\in\Theta$. The first order derivative $\frac{\partial}{\partial\theta}\lambda_i(x\;|\;\theta)$ is assumed to be bounded, continuous, and differentiable everywhere, such that the second derivative $\frac{\partial^2}{\partial\theta^2}\lambda_i(x\;|\;\theta)$ exists and is bounded. Similarly, the second derivatives $\frac{\partial^2}{\partial\theta^2}\lambda_i(x\;|\;\theta,\alpha_i)$ and $\frac{\partial^2}{\partial\alpha_{i}^2 }\lambda_i(x\;|\;\theta,\alpha_i)$ are assumed to exist and be bounded.
		\item[A$_4$:]
		The pdfs $\{f_i(x\;|\;\theta) : i\in[K]\}$ and $\{f_i(x\;|\;\theta,\alpha_i) : i\in[K]\}$ are assumed to have the same support.	
		\item[A$_5$:]
		In the setting with only the shared parameter, let us denote the FI measures under the model $i\in[K]$ for the shared parameter by
		\begin{align}
			\mathscr{I}_i(\theta)\;\triangleq\;-\mathbb{E}\Bigg[\frac{\partial^2}{\partial\theta^2}\lambda_i(x\;|\;\theta)\Bigg ]\ .		
		\end{align}
		Similarly, when we have shared and private parameters, let us denote the FI measures under the model $i\in[K]$ by
		\begin{align}
			&\mathscr{J}_i(\theta)\;\triangleq\;-\mathbb{E}\Bigg [\frac{\partial^2}{\partial\theta^2}\lambda_i(x\;|\;\theta,\alpha_i)\Bigg ]\ , \\
			\text{and}\quad
			&\mathscr{J}_{i}(\alpha_{i})\;\triangleq\;-\E\Bigg [ \frac{\partial^2}{\partial\alpha_{i}^2 }\lambda_i(x\;|\;\theta,\alpha_i)\Bigg ]\ .
		\end{align}
		We assume that the FI measures are bounded and continuous functions of $\theta$ and $\alpha_i$.
		\item[A$_6$:] We assume that the log-likelihood functions under two sufficiently distinguishable parameters $\theta$ and $\bar\theta$ are also distinguishable, that is,
		\begin{align}
			\E_\theta\bigg[\sup\limits_{\bar\theta}\lbrace \lambda_i(x\med\theta) - \lambda_i(x\med\bar\theta)\; : \; \lvert \theta-\bar\theta\rvert\;>\;\epsilon\rbrace\bigg]\;<\;0\;.
			\label{assumption:6}
		\end{align}	
		Similarly, for the setting with shared and private parameters, under sufficiently distinguishable parameters $\theta$ and $\bar\theta$, the log-likelihoods are also distinguishable, i.e.,
		\begin{align}
			\E_\theta\bigg[\sup\limits_{\bar\theta}\lbrace \lambda_i(x\med\theta,\alpha) - \lambda_i(x\med\bar\theta,\alpha)\; : \; \lvert \theta-\bar\theta\rvert\;>\;\epsilon \rbrace\bigg]\;<\;0\;,
		\end{align}	
		and, for $\alpha,\bar\alpha\in\mcA_i$ we have
		\begin{align}
			\E_\alpha\bigg[\sup\limits_{\bar\alpha}\lbrace \lambda_i(x\med\theta,\alpha) - \lambda_i(x\med\theta,\bar\alpha)\; : \; \lvert \alpha-\bar\alpha\rvert\;>\;\epsilon \rbrace\bigg]\;<\;0\;.
		\end{align}	
	\end{itemize}
	\textcolor{black}{Assumptions A$_1$, A$_2$, A$_3$, and A$_6$ collectively establish the existence of the ML estimates of the shared and private parameters. Assumptions A$_3$ and A$_5$ prove the existence of the FI measures corresponding to the shared and private parameters. Assumption A$_4$ requires every experiment to have the same support for the different parameters. We note that these assumptions are satisfied by a wide range of distributions, including continuous distributions (e.g., Gaussian and exponential) and discrete distributions (e.g., Bernoulli).}
	
	\section{Active Sampling for Shared Parameters}
	\label{section: single_param}
	In this section, we describe the active estimation procedure that solves the problem specified in (\ref{eq:var}) for the setting in which we only have the shared parameter.

	\subsection{Bayesian and Maximum Likelihood Estimators}\label{estimator}
	The goal of the optimal estimators is to minimize the conditional average cost $ {\sf C}(\hat\theta_T\;|\;\F_T)$ for given choices of $T$ and $\psi^T$ captured by $\mcF_T$. 	For $ {\sf C}(\hat\theta_T\;|\;\F_T)$ we have
	\begin{align}
		\label{eq:est1}{\sf C}(\hat\theta_T\;|\;\F_T) & = \E_t\left[\sum_{t=0}^\infty\ell(\hat\theta_t,\theta)\mathds{1}_{\{t=T\}}\;|\;\F_t\right] \\
		& = \sum_{t=0}^\infty\E_t\left[\ell(\hat\theta_t,\theta)\;|\;\F_t\right]\mathds{1}_{\{t=T\}}\label{eq:est2}\\
		& \geq  \sum_{t=0}^\infty \inf_u\;\E_t\left[\ell(u,\theta)\;|\;\F_t\right]\mathds{1}_{\{t=T\}}\; .
		\label{eq:est3}
	\end{align}
	The transition from (\ref{eq:est1}) to (\ref{eq:est2}) is due to the indicator function $\mathds{1}_{\{t=T\}}$ being $\F_t$-measurable. We denote the Bayes optimal {estimator} by
	\begin{align}
		\nu_t^{\sf B}\;\triangleq\;\arginf_u\;\E_t\left[\ell(u,\theta)\;|\;\F_t\right]\ .
	\end{align}
	Accordingly, we denote the conditional average cost associated with the Bayesian {estimate} by
	\begin{align}
		\label{eq: Bayesian_estimation_cost}
		{\sf C}_t^{\sf B}\;\triangleq\; \E_t\Big[\ell(\nu_t^{\sf B},\theta)\med\F_t\Big]\ .
	\end{align}
	Hence, from (\ref{eq:est1})-(\ref{eq: Bayesian_estimation_cost}), we have
	\begin{align}
		{\sf C}(\hat\theta_T\med\F_T)\;\geq\; \sum\limits_{t=0}^\infty {\sf C}_t^{\sf B} \cdot \mathds{1}_{\{T=t\}}\;=\; {\sf C}_T^{\sf B}\ .
	\end{align}
	This indicates that for any stopping time $T$, using Bayes optimal estimator at stopping time minimizes the estimation cost. \textcolor{black}{There are several possible choices of the cost function $\ell(\cdot,\cdot)$. We could choose $\ell(\cdot,\cdot)$ to be the maximum a-posteriori probability (MAP) cost function, i.e.,
	\begin{align}
        \ell^{\sf MAP}(\hat\theta,\theta) \triangleq \left\{
	\begin{array}{ll}
	0\ , & \mbox{if} \;\; \lVert \hat\theta-\theta\rVert \leq c\\
	1\ ,  & \mbox{if} \;\; \lVert \hat\theta-\theta\rVert>c\\
	\end{array}\right. \ ,
    \end{align}
	and, ${\sf C}^{\sf MAP}_t(\hat\theta_t\med\F_t)\triangleq \E_t[\ell^{\sf MAP}(\hat\theta_t,\theta)\med\F_t]$ is the posterior cost corresponding to the MAP cost function. As another example, $\ell(\cdot,\cdot)$ could be chosen to be the median estimation cost, i.e.,
	\begin{align}
        \ell^{\sf median}(\hat\theta,\theta)\triangleq \lVert \hat\theta-\theta\rVert_1\ , 
    \end{align}
	in which case, the posterior cost function is given by ${\sf C}^{\sf median}_t(\hat\theta_t\med\F_t)\triangleq \E_t[\ell^{\sf median}(\hat\theta_t\theta)\med\F_t]$. As mentioned, throughout this paper, we consider the minimum mean squared error (MMSE) cost function, i.e.,
	\begin{align}
        \ell(\hat\theta,\theta) \triangleq (\hat\theta-\theta)^2\ .
    \end{align}
    Corresponding to this, the conditional average posterior cost ${\sf C}(\hat\theta_t\;|\;\mathcal{F}_t)$ at time $t$ is given by ${\sf C}(\hat\theta_t\;|\;\mathcal{F}_t)\triangleq \mathbb{E}_t[(\hat\theta_t-\theta)^2\;|\;\mathcal{F}_t]$.}
	The Bayesian {estimator} under the MMSE cost function becomes the MMSE {estimator}, which we denote by 
	\begin{align}\label{eq:bayes}
		{\nu}_t^{\sf MMSE}\;\triangleq\; \E_t\big[\theta\;|\;\mcF_t\big]\;.
	\end{align}
	By specializing (\ref{eq: Bayesian_estimation_cost}) to the quadratic (MMSE) cost function and denoting the conditional average MMSE by ${\sf C}_t^{\sf MMSE}$ we have
	\begin{align}
		\label{eq:est_MMSE}
		{\sf C}(\hat\theta_T\med\F_t)\;\geq\; \sum\limits_{t=0}^\infty {\sf C}_t^{\sf MMSE}\cdot\mathds{1}_{\{T=t\}}\;=\; {\sf C}_T^{\sf MMSE}\ ,
	\end{align}
	rendering the MMSE \textcolor{black}{estimate} as the optimal \textcolor{black}{estimate} under the quadratic cost function. Besides the MMSE {estimator} $\nu_t^{\sf MMSE}$, we also use the ML {estimator}  for designing our algorithm. Specifically, the MMSE and the ML \textcolor{black}{estimates} each serve a specific purpose. We use the MMSE \textcolor{black}{estimate} for producing the final estimates for the parameters of interest, and use the ML \textcolor{black}{estimate} for guiding the sampling decisions. We denote the ML \textcolor{black}{estimator} of $\theta$ by
	\begin{align}
		\nu_t^{\sf ML}\;\triangleq\; \color{black}{\argmax\limits_{\theta\in\Theta}}\; \sum\limits_{i\in\psi^t} \lambda_i(Y_i\med\theta)\ .
	\end{align}
	\subsection{Chernoff-like Sampling Rule}
	\label{Control}
	Our sampling rule follows the spirit of the sequential experimental design due to Chernoff \cite{Chernoff}, which addresses the problem of active sampling for binary composite hypothesis testing. Under the assumption of uniformly distinguishable hypothesis and a finite set of control actions, at each round, Chernoff's rule decides in favor of the design that maximizes the \textit{immediate} return. Such return, in the context of hypothesis testing, is characterized by a function of the Kullback-Leibler (KL) divergence of the models under different hypotheses. Specifically, Chernoff's rule determines the maximum likelihood (ML) decision about the most likely hypothesis at each instant, and then chooses an action that maximally reinforces this decision. 
	\par In the context of sequential estimation, at each time step $t\in\N$, we wish to select the experiment that results in the \textit{most informative} observation, that is the one which is likely to produce the largest reduction in the estimation cost. As a relevant measure for comparing the informativeness of various experiments in the sequential estimation framework, we adopt the FI measure. Specifically, for selecting the experiment $\psi({t})\in[K]$, we compute the ML estimate generated by $\nu_{t-1}^{\sf ML}$ based on the sequence of samples accumulated up to time $t-1$. We then select the experiment that maximizes the FI measure computed at the ML estimate, i.e., 
	\begin{align}\label{eq:psi}
		\psi(t)\;=\;\argmax\limits_{i\in[K]}\;\mathscr{I}_i(\nu_{t-1}^{\sf ML})\;,
	\end{align}
	where a potential tie is broken by selecting one uniformly at random. This sampling rule is greedy in the sense that it only focuses on exploiting the most informative experiment. We will show that this rule is optimal only for estimating the shared parameter. It loses its optimality when we also have private parameters, for which we provide an alternate sampling rule in Section \ref{FullyUnknown}.
	\subsection{Stopping Rule}
	\label{Stop}
	Finally, we specify the stopping rule that characterizes the end of the sampling procedure. The rule  is directly driven by the decision quality constraint specified in the formulation of problem $\mcP(\beta)$ in~\eqref{eq:var}. Specifically, based on $\mcP(\beta)$, we are interested in minimizing the number of samples such that the average posterior estimation cost falls below the target reliability threshold $\beta$. Thus, we set the stopping time as the first time that the cost $ {\sf C}(\nu_T^{\sf MMSE}\;|\;\F_T) $ falls below $\beta$, i.e., 
	\begin{align}\label{eq:N}
		T\;\triangleq\; \inf\;\big\lbrace t\in\N\; :\;  {\sf C}(\nu_t^{\sf MMSE}\med\F_t) \;\leq\;\beta \big\rbrace\;.
	\end{align}
	The structure of (\ref{eq:N}) is similar to that of~\cite{MoustakidesISIT}, with the key difference that the posterior variance ${\sf C}(\nu_T^{\sf MMSE}\;|\;\F_T)$ not only depends on the estimator, but also on the sampling path $\psi^T$, which does not exist in~\cite{MoustakidesISIT}.
	
	\subsection{Performance Guarantees}
	In this section, we evaluate the optimality of the active sampling procedure for sequential estimation. First, we provide a universal lower bound on the average sample complexity of any procedure that solves (\ref{eq:var}). Next, we provide a high-probability upper bound on the sample complexity achieved by our proposed procedure for a range of the prescribed guarantee~$\beta$. Subsequently, we remark that the average sample complexity of our proposed procedure achieves the lower bound in the asymptote of small values of $\beta$. 
	To this end, corresponding to $\theta$, we define $V_{\beta}(\theta)$, which is instrumental in characterizing the sample complexity of our sequential estimation problem:
	
	\begin{align}
	\label{eq:v_beta}
		V_{\beta}(\theta)\;\triangleq\;\inf_{\bp\in\Q_K}\; \frac{1}{\beta}\Bigg(\sum\limits_{i\in[K]}p(i)\mathscr{I}_i(\theta)\Bigg)^{-1}\;,
	\end{align}
	where $\Q_K$ denotes the $K$-dimensional probability simplex and we have defined $\bp\triangleq [p_1,\cdots,p_K]$. \textcolor{black}{It can be readily verified that under the assumptions A$_1$-A$_6$, (\ref{eq:v_beta}) can be simplified to $V_\beta(\theta) = \min_{i\in[K]}\frac{1}{\beta \mathscr{I}_i(\theta)}$.} 
	Based on this definition, we now provide a lower bound on the average sample complexity. Note that for Theorem \ref{theorem:SC_Lower_bound_single_param} and all the other subsequent theorems, expectation $\E$ and probability $\P$ measures are with respect to the measures induced by the randomness in the observations, the control actions, and the stochastic stopping rule.
	\begin{theorem}[Converse]
		\label{theorem:SC_Lower_bound_single_param}
		Under assumptions A$_1$-A$_6$, for any sequential procedure $\Delta$, and for any $h>0$, there exists a constant $C(h)>0$ such that for any $\beta\in(0,C(h))$, we have
		\begin{align}
			\E[T]\;\geq\; \Bigg(V_{\beta}(\theta)-\frac{h}{\beta}\Bigg)(1-h)\;.
			\label{eq:lower_bound_SC_single_param}	
		\end{align} 
	\end{theorem} 
	\begin{proof}
		See appendix \ref{appendix: sc_lower_bound_single_param}.
	\end{proof}
	
	\begin{theorem}[Achievability]
		\label{theorem: SC_upper_bound_single_param}
		Under assumptions A$_1$-A$_6$, for any $h>0$, there exists a constant $C^\prime(h)>0$ such that for any $\beta\in(0,C^\prime(h))$, the proposed procedure achieves
		\begin{align}
			\P\Bigg\lbrace T\;\leq\; V_{\beta}(\theta) + \frac{h}{\beta} + 1\Bigg\rbrace\;=\;1\;.
			\label{eq:SC_upper_bound_single_param}
		\end{align}
	\end{theorem}
	\begin{proof}
		See appendix \ref{proof: SC_upper_bound_single_param}.
	\end{proof}
	\begin{theorem}[Achievability]
		\label{remark:ach_single_param}
		In the asymptote of $\beta\rightarrow 0$, the proposed procedure satisfies \begin{align}
		\label{eq:SC_UB_new}
			\lim\limits_{\beta\rightarrow 0}\; \frac{\E[T]}{V_{\beta}(\theta)}\;\leq\;1\;.
		\end{align}
	\end{theorem}
	\begin{proof}
		See Appendix \ref{proof: avg_single_ach}.
	\end{proof}

	\textcolor{black}{Note that the upper bound on the average delay provided in Theorem~\ref{remark:ach_single_param} matches the universal lower bound on the average delay obtained in Theorem~\ref{theorem:SC_Lower_bound_single_param}. Specifically, since Theorem~\ref{theorem:SC_Lower_bound_single_param} holds for any $h>0$, we can take the supremum over $h$, followed by the limit with respect to $\beta$. Thus, in the asymptote of $\beta\rightarrow 0$, the lower-bound on the average sample complexity specified in~(\ref{eq:lower_bound_SC_single_param}) becomes $\limsup_{\beta\rightarrow 0}\E[T]/V_{\beta}(\theta)\geq 1$, which is also the upper-bound on the average sample complexity of the proposed procedure, specified in~(\ref{eq:SC_UB_new}). This establishes the optimality of the proposed procedure $\Delta$, in the asymptote of a diminishing guarantee on the estimation cost.}
	
	Furthermore, note that under the presence of a single control, when $K=1$, our procedure reduces to that of \cite{MoustakidesISIT}. \textcolor{black}{We note that~\cite{MoustakidesISIT} addresses a problem that is a special case of the problem we consider in two ways. First, it focuses only on one data stream (experiment). Secondly, which is also an artifact of the first point, in \cite{MoustakidesISIT} there is no notion of active selection of the experiments/streams. By setting $K=1$, our asymptotic bound on the average sample complexity provided in Theorem 3 reduces to the result of~\cite{MoustakidesISIT}. Specifically, it can be readily verified that for
    the case of $K=1$, we have
    \begin{align}
        \lim\limits_{\beta\rightarrow 0}\frac{\E[T]}{(1/\beta)}\leq \mathscr{J}(\theta)^{-1}\ .
    \end{align}
    Although~\cite{MoustakidesISIT} does not provide any expression for the average delay, it proves the optimality of the proposed rules $(\nu_T^{\sf MMSE},T)$, which are the same \textcolor{black}{estimate} and stopping rule that we specified in (34) and (38).} Furthermore, for $K=1$, it is shown in~\cite{MoustakidesISIT} that the proposed procedure achieves optimality in all regimes (both asymptotic and non-asymptotic, for all values of $\beta$). This is due to the fact that when $K=1$, the control action has only one experiment to choose from.

	\section{Active Sampling for Shared and Private Parameters}\label{FullyUnknown}
	In this section, we extend the active sampling procedure and the attendant performance guarantees to the setting in which the experiments have both shared and private parameters.  
	\subsection{Bayesian and Maximum Likelihood Estimators}
	\label{sec:estimator}
	Following the same line of arguments as in (\ref{eq:est1})-(\ref{eq:est_MMSE}), we use the MMSE estimator for minimizing the average posterior conditional MMSE estimation costs ${\sf C}(\hat\theta_T\med\F_T)$ and ${\sf D}(\hat\alpha_{i,T}\med \F_T, \hat\theta_T)$ defined in (\ref{eq:cost_C}) and (\ref{eq:cost_D}), respectively. Accordingly, we denote the \textcolor{black}{estimators}  for the parameters $\theta$ and $\{\alpha_i : i\in[K]\}$ by 
	\begin{align}\label{eq:est_multi}
		{\nu}_t^{\sf MMSE}\;\triangleq\; \E_{t}\big[\theta\;|\;\mcF_t\big]\;,\;\text{and}\;{\zeta}_{i,t}^{\sf MMSE}\;\triangleq\; \E_{t}^i\big[\alpha_i\;|\;\mcF_t,\;\nu_t^{\sf MMSE}\big]\;.
	\end{align}
    Furthermore, in our sampling rule, we leverage the ML estimates of the parameters $\theta$ and $\{\alpha_i : i\in[K]\}$. For this purpose, let $\phi_i\triangleq[\theta,\alpha_i]$ denote the vector containing the shared and the private parameters. Furthermore, we denote the ML \textcolor{black}{estimators}  by
	\begin{align}
		&\nu_t^{\sf ML}\;\triangleq\; \color{black}{\argmax\limits_{\theta\in\Theta}} \sum\limits_{i\in\psi^t}  \lambda_i(Y_i\med\theta,\balpha) \ ,\\
		&\zeta_{i,t}^{\sf ML}\;\triangleq\; \color{black}{\argmax\limits_{\alpha_i\in\mcA_i}} \sum\limits_{s\in\psi_i^t}  \lambda_i(Y_s\med \theta,\alpha_i)\ ,\\
		\text{and}\quad & \phi_{i,t}^{\sf ML}\;\triangleq\; \color{black}{\argmax\limits_{\phi\in[\theta\times\mcA_i]}} \sum\limits_{s\in\psi^t_i}  \lambda_i(Y_s\med\phi)\ ,
	\end{align}
	where we have defined $\psi_i^t\triangleq \big\lbrace t\in\{1,\cdots, t\} : \psi(t)=i\big\rbrace$ as the ordered sequence of time instants during $\{1,\cdots, t\}$ at which experiment $i\in[K]$ is selected for sampling. 
	\subsection{Sampling Rules}\label{sec:ControlPolicy2}
	To accommodate the distinct levels of tolerance for the estimation costs associated with the shared and the private parameters, we need a sampling rule that is adaptive to the thresholds $\bbeta$ imposed on the cost functions. Before we formally specify our sampling rule, we discuss an adaptation of the greedy sampling rule that we used in Section~\ref{Control}, and show that such a greedy approach becomes sub-optimal in this setting, caused by insufficient exploration.    
	\subsubsection{Greedy Sampling Rule}\label{chernoff_policy}
	The sampling rule described in Section~\ref{Control} aims at selecting the experiment that maximizes the FI measure at the current ML estimate. By generalizing this approach, a greedy sampling strategy aims at selecting the experiment that maximizes the FI measure for the parameter $\phi_i$ for all $i\in[K]$. This can be formalized as:
	\begin{align}\label{eq:chern_wrong}
		\psi^{\sf c}(t) \;\triangleq\; \argmax_{i\in[K]}\; \trace\bigg\lbrace\mathscr{I}_i({\phi}_{i,t}^{\sf ML})\bigg\rbrace\;,
	\end{align}
	where $\mathscr{I}_i(\phi_i)$ represents the Fisher information matrix (FIM) associated with $\phi_i$, for $i\in[K]$. This sampling rule does not use the information that the guarantees required for different estimation qualities associated with different experiments may not be the same. This becomes problematic when the experiment that maximizes the FI measure requires a considerably weaker guarantee
	(i.e., it has a large value of tolerance) on the estimation cost. However, due to the sampling rule being agnostic to the tolerance levels, the greedy rule in (\ref{eq:chern_wrong}) continues sampling the same experiment, even after the target estimation quality is achieved. This renders the greedy sampling rule prone to insufficient exploration. This is stated more formally in the following theorem. 
	\begin{theorem}\label{remark:chernoff}
		There exists $\beta_i>0$ for any experiment $i\in[K]$, such that any sequence of sequential procedure $\Delta^{\sf c}$ that involves a Chernoff-based control action $\psi^{\sf c}(t)$ defined in (48) is sub-optimal.
	\end{theorem}
	\begin{proof}
		See Appendix \ref{R1}. 
	\end{proof}
	This theorem shows that a greedy sampling strategy based on selecting the most informative experiment, along with any choice of stopping rule that satisfies the constraints in~(\ref{eq:var_multi}) at stopping, renders an infinite average sample complexity for specific choices of the estimation guarantees.

	\subsubsection{Cost-aware Sampling Rule} 
	Next, we propose a sampling rule that maintains a balance between exploiting the control actions that maximize the return and exploring actions that have not yet been sufficiently sampled. This ensures that the algorithm does not get stuck in using only the most informative experiments. This is critically needed to avoid insufficient exploration of the less informative experiments, since this leads to significant disparity among different estimation qualities. The more explored experiments will be over-sampled, achieving estimation qualities stronger than the prescribed thresholds. This penalizes the overall sample complexity of the sampling process. To circumvent such oversampling, we propose a sampling rule that involves randomly sampling from the distribution defined as
	\begin{align}
		\bar{q}_t\; \triangleq\; \textcolor{black}{\arg\min}_{q\in\Q_K}\bigg\lbrace &\frac{1}{\beta}\Big (\sum\limits_{i\in[K]} q(i)\mathscr{J}_i({\nu}_t^{\sf ML})\Big )^{-1} \nonumber\\&\qquad+
		\sum\limits_{i\in[K]}\frac{1}{\beta_i} \Big (q(i)\mathscr{J}_i({\zeta}_{i,t}^{\sf ML})\Big )^{-1}\bigg\rbrace\ .
		\label{eq:Control2}
	\end{align}
	Note that $\bar{q}_t$ does not necessarily place the entire mass on \emph{one} of the control actions, thus facilitating exploration. Furthermore, the distribution $\bar q_t$ converges to a limiting distribution in the limit of $t\rightarrow\infty$. This is attributed to the fact that the FI measures are computed at the ML estimates, and, by the strong consistency of the ML estimates~\cite{MLE}, $\nu_t^{\sf ML}\xrightarrow{a.s.}\theta$, and $\zeta_{i,t}^{\sf ML}\xrightarrow{a.s.}\alpha_i$ for every $i\in[K]$. Thus, as $t\rightarrow\infty$, the ML estimates converge to the respective ground truths $\theta$ and $\{\alpha_i : i\in[K]\}$, and $\bar q_t$ converge to its limiting distribution. We denote this limiting distribution by
	\begin{align}
		q^* \triangleq \textcolor{black}{\arg\min}_{q\in\Q_K}\; \bigg\lbrace &\frac{1}{\beta}\Big (\sum\limits_{i\in[K]} q(i)\mathscr{J}_i(\theta)\Big )^{-1} \\
\nonumber\\&\qquad+
		\sum\limits_{i\in[K]}\frac{1}{\beta_i}\Big (q(i)\mathscr{J}_i(\alpha_i)\Big )^{-1}\bigg\rbrace\;.
		\label{Control2_optimal}
	\end{align} 
	\subsection{Stopping Rule}\label{Stopping Rule 2}
	We design a stopping rule that takes into account the fidelity guarantees on the shared parameter and the private parameters. Specifically, at each time instant, we compute the conditional posterior MMSE cost of the estimated shared parameter ${\sf C} (\nu_t^{\sf MMSE}\;|\;\F_t)$ and the private parameters $\{{\sf D} (\zeta_{i,t}^{\sf MMSE}\;|\;\F_t,\;\nu_t^{\sf MMSE}) : i\in[K]\}$. Our proposed stopping rule $T$ is given by
	
	\begin{align}\label{eq:stop2}
		T\;\triangleq\; \inf\;\Big\lbrace &t\in\N\;:\; {\sf C} (\nu_t^{\sf MMSE}\;|\;\F_t)\;\leq\;\beta,\nonumber\\&\quad {\sf D} (\zeta_{i,t}^{\sf MMSE}\;|\;\F_t,\;\nu_t^{\sf MMSE})\;\leq\;\beta_i, \;\forall\;i\in[K]  \Big\rbrace\ .
	\end{align}
	Based on this, the sampling process terminates at the first instant at which all estimation cost constraints are satisfied.
	\subsection{Performance Guarantees}
	In this section, we present the optimality guarantees of the proposed active sampling procedure. We begin by stating a lower bound on the average sample complexity for any algorithm that solves $\mcP(\bbeta)$, defined in (\ref{eq:var_multi}). Next, we provide an upper bound on the average sample complexity of the proposed sequential procedures. Specifically, we provide a high probability upper bound on the sample complexity as well as an asymptotic upper bound on the average sample complexity. The average sample complexity is shown to achieve the universal lower bound asymptotically up to a constant term. 
	To characterize the average sample complexity, we define
	\begin{align}
		W_{\bbeta}(\theta,\balpha)\;\triangleq\; \inf\limits_{q\in\Q_K}\; \bigg\{\frac{1}{\beta}\bigg(&\sum\limits_{i\in[K]}q(i)\mathscr{J}_i(\theta)\bigg)^{-1} \nonumber\\&\quad+ \sum\limits_{i\in[K]}\frac{1}{\beta_i}\Big(q(i)\mathscr{J}_i(\alpha_i)\Big)^{-1}\bigg\}\ .
	\end{align}
	Next, we define
	\begin{align}
		\label{eq:Vs}
		&\color{black}{V(\theta)}\;\triangleq\; \frac{1}{\beta}\Bigg(\sum\limits_{i\in[K]}q^*(i)\mathscr{J}_i(\theta)\Bigg)^{-1}\ ,\nonumber\\ \text{and}\quad &V_i(\alpha_i)\;\triangleq\; \frac{1}{\beta_i}\Bigg(q^*(i)\mathscr{J}_i(\alpha_i)\Bigg)^{-1}\ ,\;\forall\;i\in[K]\ .
	\end{align}
	Accordingly, we define define
	\begin{align}
		&V_{\max}(\theta,\balpha)\;\triangleq\;\max\bigg({V(\theta)},{V_1(\alpha_1)},\cdots,{V_K(\alpha_K)}\bigg)\ ,\nonumber\\ \text{and}\quad &V_{\min}(\theta,\balpha)\;\triangleq\;\min\bigg({V(\theta)},{V_1(\alpha_1)},\cdots,{V_K(\alpha_K)}\bigg ) \ .
	\end{align}
	Furthermore, define $\beta_{\max}$ and $\beta_{\min}$ as the maximum and minimum tolerance levels on the estimation costs, i.e.,
	\begin{align}
		\beta_{\max}\triangleq\max\{\beta,\beta_1,\cdots,\beta_K\},\;\;\text{and}\;\; \beta_{\min}\triangleq\min\{\beta,\beta_1\cdots,\beta_K\}\ .
	\end{align}
	Based on these definitions, we provide a lower bound on the average sample complexity.
	\begin{theorem}[Converse]
		\label{theorem:multi_sc_converse}
		Under assumptions A$_1$-A$_6$, for any sequential procedure $\Delta$, and for any $h>0$ there exist constants $C(h)>0$ and $\{D_i(h) > 0 : i\in[K]\}$, such that for any $\beta\in(0,C(h))$ and $\beta_i\in(0,D_i(h))$ for all $i\in[K]$, we have
		\begin{align}
			\E[T]\;\geq\; \frac{W_{\bbeta}(\theta,\balpha)-h}{K+1}\cdot\Big(1-(K+1)h\Big)\;.
		\end{align}
	\end{theorem}
	\begin{proof}
		See Appendix ~\ref{proof:multi_sc_converse}.
	\end{proof}
	\begin{theorem}[Achievability]
		\label{theorem:multi_sc_ach}
		Under assumptions A$_1$-A$_6$, for any $h>0$ there exist constants $L(h)>0$ and $\{M_i(h) > 0 : i\in[K]\}$, such that for any $\beta\in(0,L(h))$ and $\beta_i\in(0,M_i(h))$ for all $i\in[K]$, the proposed procedure comprised of the sequence of \textcolor{black}{estimates} in~(\ref{eq:est_multi}), sampling rule in~(\ref{eq:Control2}) and the stopping rule in~(\ref{eq:stop2}) satisfies
		\begin{align}
			\P\Bigg\{T\leq \frac{W_{\bbeta}(\theta,\balpha)}{K+1} &+ \Big(V_{\max}(\theta,\balpha)-V_{\min}(\theta,\balpha)\Big)\nonumber\\&\qquad+\bigg(\frac{1}{\beta_{\min}} + \frac{1}{\beta_{\max}}\bigg)h + 2\Bigg\}\;=\; 1\ .
		\end{align}
	\end{theorem}
	\begin{proof}
		See Appendix ~\ref{proof:multi_sc_ach}.
	\end{proof}
	\begin{theorem}[Achievability]
		\label{theorem:multi_ach_avg}
		In the asymptote of $\bbeta\rightarrow \mathbf{0}$, there exists $h\in(0,1)$ such that the proposed procedure comprised of the sequence of \textcolor{black}{estimates} in~(\ref{eq:est_multi}), sampling rule in~(\ref{eq:Control2}), and the stopping rule in~(\ref{eq:stop2}) satisfies
		\begin{align}
			\lim\limits_{\bbeta\rightarrow \mathbf{0}}\frac{\E[T]}{W_{\bbeta}(\theta,\balpha)}\;\leq\; \frac{1}{K+1} + h\;.
		\end{align}
	\end{theorem}
	\begin{proof}
		See Appendix~\ref{proof:multi_ach_avg}.
	\end{proof}
	
	\section{Numerical Experiments}\label{Simulation}

\subsection{Shared Parameter}
Consider a network of $K$ sensors, where $K$ is an even number. Sensor $i\in[K]$ generates its samples according to the distribution $\mathcal{N}(0,\sigma^2_i(\theta))$, where
\begin{align}
    \sigma^2_i(\theta) \triangleq \left\{
    \begin{array}{ll}
          \frac{(i-1)^2}{K(K-i+1)} + \frac{K-2i+2}{K-i+1}\cdot \theta\ , & \forall i\in\left\{1,\dots,\frac{K}{2}\right\}\\
          &\\
          \frac{i}{K} + \frac{K-2i}{i} \cdot\theta\ , & \forall i\in\left\{\frac{K}{2}+1,\dots K\right\}\;
    \end{array}\right. \ .
\end{align}
The choice of the variance values for the Gaussian distribution corresponding to each sensor makes them the \textit{most informative} one under a specific regime of the underlying shared parameter $\theta$. More specifically, note that the FI measure for the Gaussian distribution $\mathcal{N}(0,\sigma^2)$ with respect to $\sigma$ is $\mathcal{I}\triangleq 1/2\sigma^4$. It can be readily verified that each sensor maximizes the FI under the following regimes of $\theta$: for $\theta\in((i-1)/K\;,\;i/K)$ for $i\in[K]$, sensor $i$ maximizes the FI. 
In our evaluations, we set $K=4$. 
$\theta$ is assumed to have a uniform prior distribution ${\sf Unif}[0.01,0.99]$. It can be readily verified that if the true value of $\theta$ is less than $0.25$, Sensor $1$ is the most informative sensor. Otherwise, if $0.25<\theta< 0.5$, Sensor $2$ becomes the most informative one. Similarly, for $0.5<\theta<0.75$ and $0.75<\theta<1$, Sensor $3$ and Sensor $4$, respectively, become the most informative ones. Thus, an effective sampling rule is characterized by its ability to identify and converge to the best sensor using as few samples as possible. For our experiment, we set $\theta=0.2$. 
\par Figure \ref{fig:N_vs_beta} shows the average number of samples required $\mathbb{E}[T]$ versus various levels of tolerance $\beta$, and compares them against the following four approaches for sensor selection. 
\begin{itemize}
	\item[1.] \textit{Random selection:} Random selection forms a baseline for comparison. Essentially, it refers to sampling one of the sensors $S_1$, $S_2$, $S_3$, or $S_4$ uniformly at random. The same stopping rule specified in (\ref{eq:N}) is used for fair comparison.
	\item[2.] \textit{Genie-aided sampling:} In this setting, we consider a genie-aided scenario in which the sampling rule is informed what the most informative sensor for estimating the unknown $\theta$ is.
	\item[3.] \textit{Approach of \cite{Atia}:} The algorithm prescribed in \cite{Atia} proposes a different stopping rule based on a fixed cost of sampling $c$, while keeping the same sensor-selection policy. The approach trades off the estimation performance against the accumulated cost of sampling, and thus does not have an explicit performance guarantee on the estimation cost (or a counterpart of $\beta$ in our setting). For comparison, for any given $\beta$, we find out a value of $c$ that ensures that the estimation cost of \cite{Atia} falls below $\beta$, and use that to generate the variations of $\E[T]$ versus~$\beta$.
\end{itemize} 
The performance shown in Figure~\ref{fig:N_vs_beta} correspond to averaging over~$100$ Monte Carlo realizations. Note that in Figure~\ref{fig:N_vs_beta}, Chernoff-like greedy sampling corresponds to the proposed sampling rule, which is also a special case of the look-ahead active sampling rule proposed in (\ref{eq:Control2}), for the case that there is only the shared parameter. This figure shows that our sampling rule outperforms the random sampling strategy, and the procedure prescribed in \cite{Atia}. \textcolor{black}{Note that the genie-aided sampling strategy uses the fewest number of samples to meet the prescribed level $\beta$. This is because this sampling strategy knows the most informative sensor from $t=1$ and always samples from that sensor. On the other hand, our proposed Chernoff-like greedy sampling rule requires a few more samples to guess the most informative sensor before it starts drawing samples from that sensor. Thus, its sample complexity is worse than that of the genie-aided strategy but better than those of all other strategies. Furthermore, while the approach of~\cite{Atia} uses the same sampling strategy like ours, its sample complexity suffers due to the choice of stopping rule. Specifically, the stopping rule in~\cite{Atia} is designed to minimize a unified objective comprising the estimation cost and delay, which is different from the objective in~(\ref{eq:var}). Finally, the random sampling strategy puts equal sampling effort on each sensor, thus, requiring a more significant average number of samples to reach the same guarantee on the estimation cost.}

\textcolor{black}{Furthermore, to gain more insight regarding the scaling behavior of the proposed algorithm with respect to the number of sensors $K$, Figure~\ref{fig:N_vs_K} plots the average number of samples against $K$. Clearly, as the number of sensors increases, the number of samples required by the proposed strategy in identifying the most informative sensor increases, thus, increasing the average sample complexity. For this experiment, we set $\beta=0.005$, and the other parameters remain the same.}

\begin{figure}[htb]
	\centering 
	\includegraphics[width=0.8\linewidth]{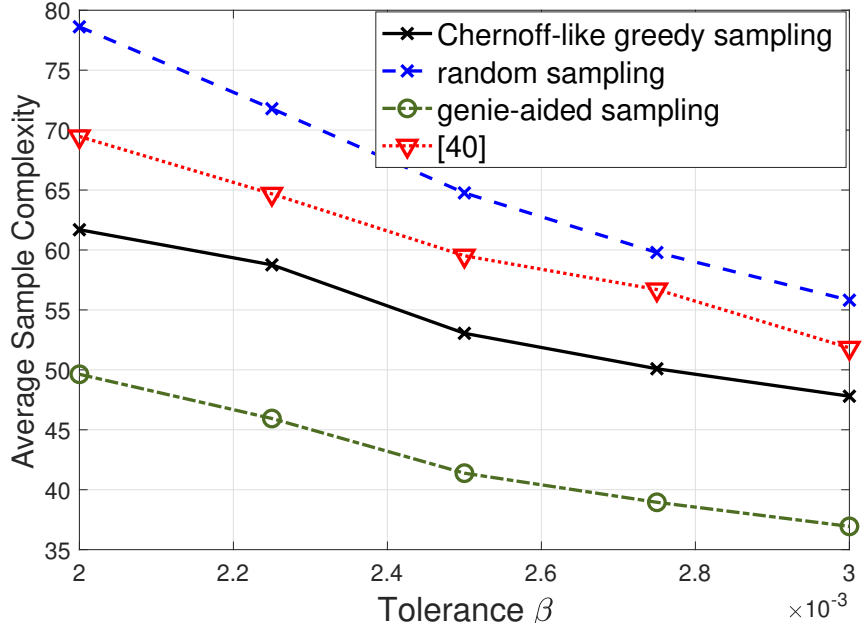} 
	\caption{Average sample complexity versus  prescribed tolerance $\beta$.}
	\label{fig:N_vs_beta} 
\end{figure}

\begin{figure}[htb]
	\centering 
	\includegraphics[width=0.8\linewidth]{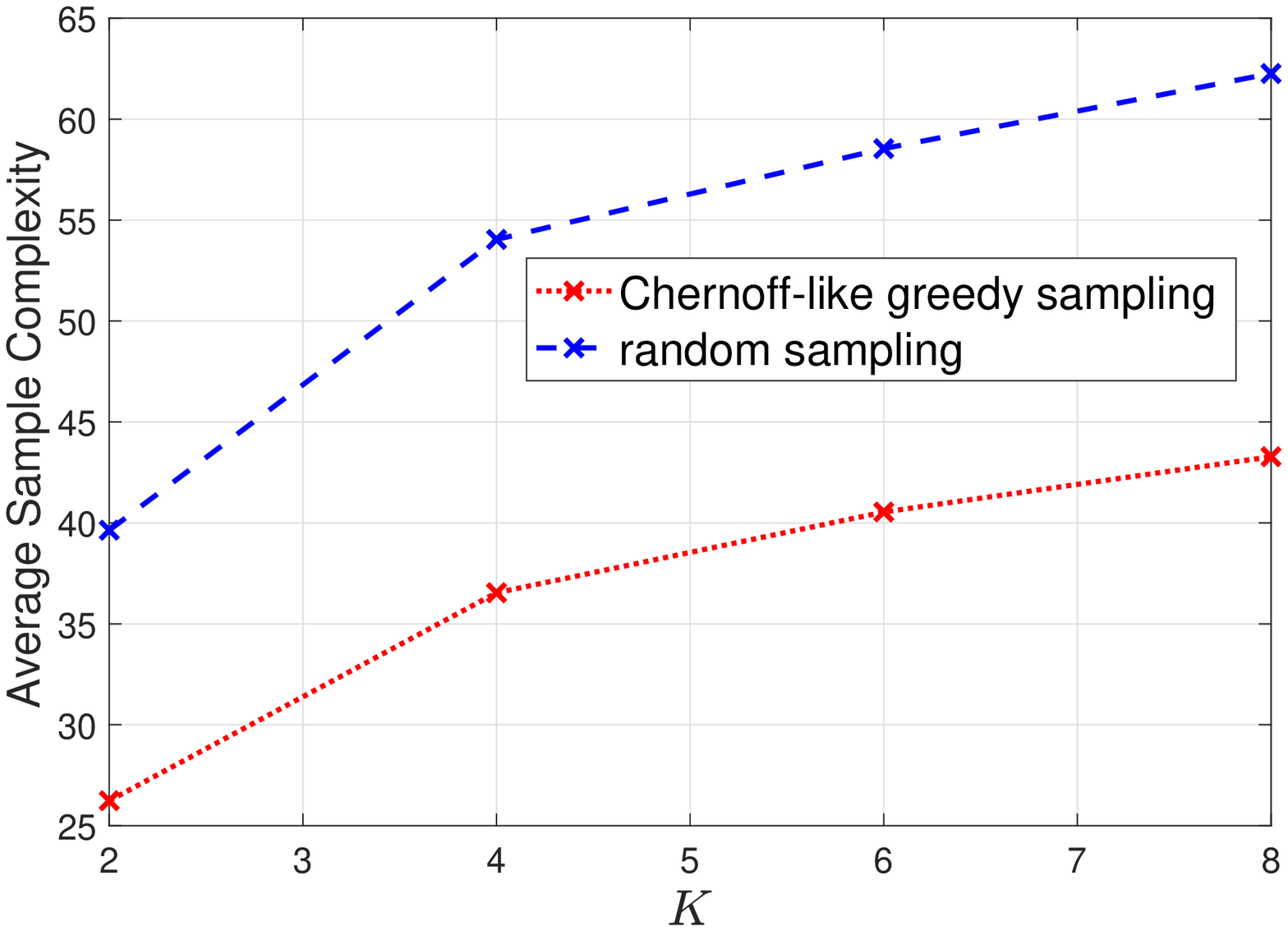} 
	\caption{Average sample complexity versus varying number of sensors $K$.}
	\label{fig:N_vs_K} 
\end{figure}

\subsection{Shared and Private Parameters}
The second experiment showcases the performance of the active sequential estimation algorithm proposed in Section~\ref{FullyUnknown}. We consider a simple network of two sensors, each generating random samples of~$\theta$ contaminated by noise. Each sensor is subject to a different level of noise variance. Sensors $1$ and $2$ have distributions $\mathcal{N}(\theta,\alpha_1)$ and $\mathcal{N}(\theta,\alpha_2)$, respectively. We assume a uniform prior for the mean, i.e., for given $a_0, b_0\in\R$,  
\begin{align}
	\pi_{\theta}(\theta) = \frac{1}{b_0-a_0}\mathds{1}_{\lbrace a_0\leq\theta\leq b_0\rbrace}\;, \quad b_0>a_0>0\ .
\end{align}
Similarly, the variance of each sensor has a uniform prior, i.e.,
\begin{align}
	&\pi_i(\alpha_i) = \frac{1}{b_i-a_i}\mathds{1}_{\lbrace a_i\leq\alpha_i\leq b_i\rbrace}\;,\quad b_i>a_i>0\ ,\;i\in\{1,2\}\ .
\end{align}  
The posterior distributions of the unknown mean and variance are analytically intractable and they are computed numerically. The performance of our proposed procedure for this setup is depicted in Fig. \ref{fig:N_vs_beta_learn}. For this experiment, we have set $\theta=0.25$, $\alpha_1 = 0.25$, and $\alpha_2=4$. The corresponding priors are parameterized by $a_0 = 0.1$, $b_0=4$, $a_1 = 0.1$, $b_1 = 0.7$, $a_2 = 1$, and $b_2=5$. The confidence levels on the estimates of the variance $\alpha_1$ and $\alpha_2$ are fixed at $\beta_1 = 0.1$ and $\beta_2 = 0.05$. 
The results show that the proposed sampling strategy outperforms the random selection strategy, which we use as a baseline in this case, as well as the greedy sampling strategy described in Section~\ref{chernoff_policy}. This matches our theoretical analysis, where we prove in Appendix \ref{R1} that the Chernoff-based greedy sampling rule is sub-optimal. Intuitively,
it is clear from the setting that such a sampling strategy focuses on exploiting Sensor $1$ in the long run. However, that could result in an insufficient exploration of Sensor $2$, resulting in a bad estimate of the variance of the second sensor.

\begin{figure}[htb]
	\centering 
	\includegraphics[width=0.8\linewidth]{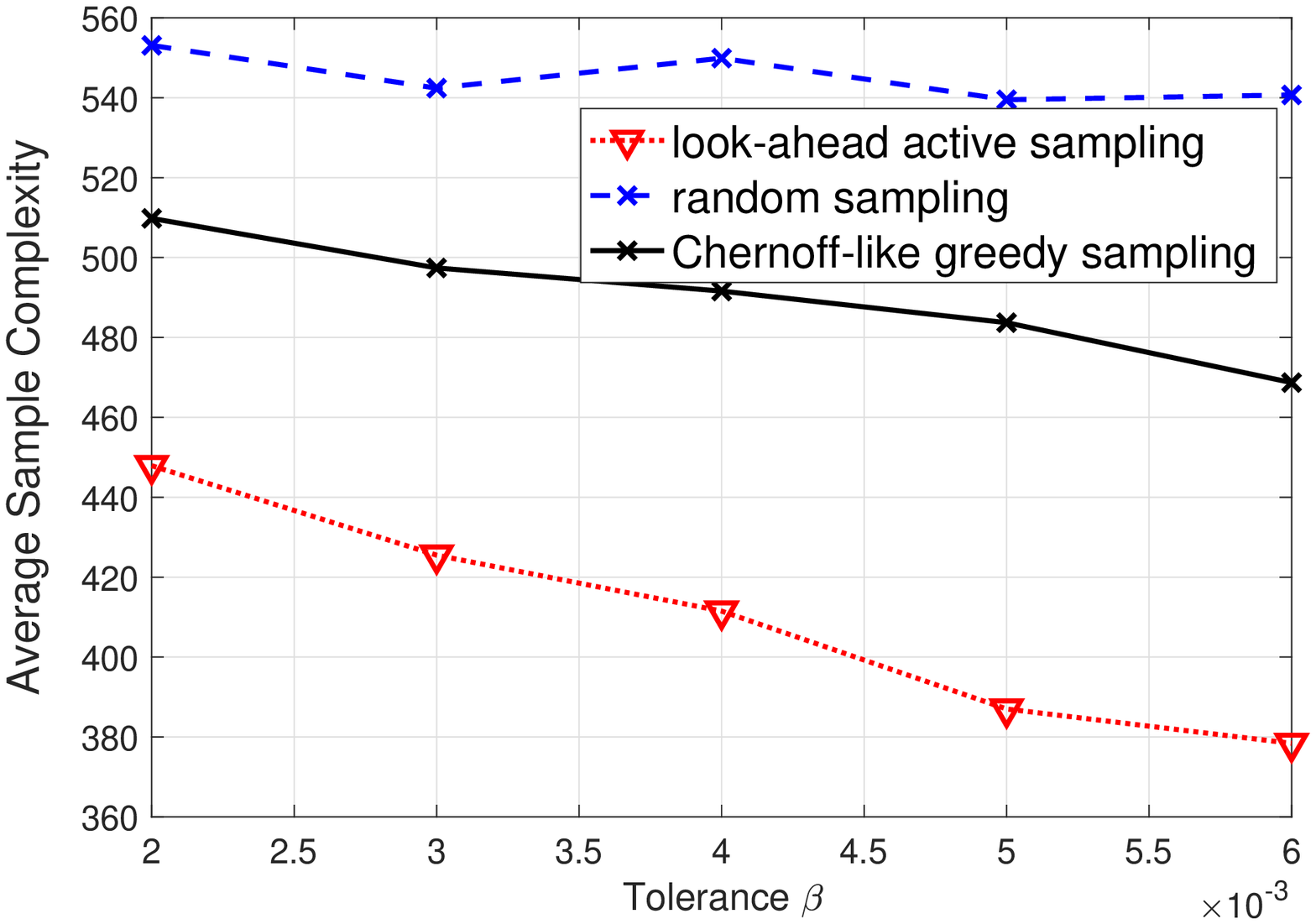} 
	\caption{Average sample complexity versus $\beta$, $\beta_1=0.05$, and $\beta_2=0.1$.}
	\label{fig:N_vs_beta_learn} 
\end{figure}

\subsection{Real-world Data}

\textcolor{black}{In this section, we provide an experiment on a real-world dataset, comparing the performance of the proposed look-ahead sampling procedure against the random selection strategy. For this purpose, we use the Chicago beach weather station dataset, which consists of three different weather stations recording hourly measurements on various parameters, such as the air temperature, humidity, rain intensity, wind direction, and wind speed. We select two of these weather stations, namely the $63^{\sf rd}$ Street weather station and the Foster weather station. Our goal is to estimate the average air temperature in the month of September at 10 AM from noisy measurements of the temperature. We set the ground truth of the air temperature to the average temperature recorded by the two weather stations in September at 10 AM over the years $2019$, $2020$, and $2021$. The true variance for each weather station is also set to the variance computed from the data over these years. It is noteworthy that the Foster weather station has a larger variance of $6.7975$ in recorded temperatures, compared to the $63^{\sf rd}$ Street weather station, which has a variance of $4.1216$. We consider the setting in which the variance of the weather station sensors is also unknown (in addition to the temperature that we intend to estimate). Each station is assumed to generate noisy measurements of the temperature, drawn from a Gaussian distribution with mean and variance set as specified. We set the levels required on the cost of estimating the variance for both the weather stations to $\beta_1 = \beta_2 = 0.1$. The prior distributions of the mean values are assumed to be ${\sf Unif}[10,20]$, while those of the variances are assumed to be ${\sf Unif}[2,8]$ for both the weather stations. Figure~\ref{fig:RW} demonstrates the average sample complexity against various levels of $\beta$. We observe that our proposed sequential procedure yields a far superior performance compared to random sampling, thus, clearly depicting the advantage of our proposed active sampling procedure.} 

\begin{figure}[t]
	\centering 
	\includegraphics[width=0.8\linewidth]{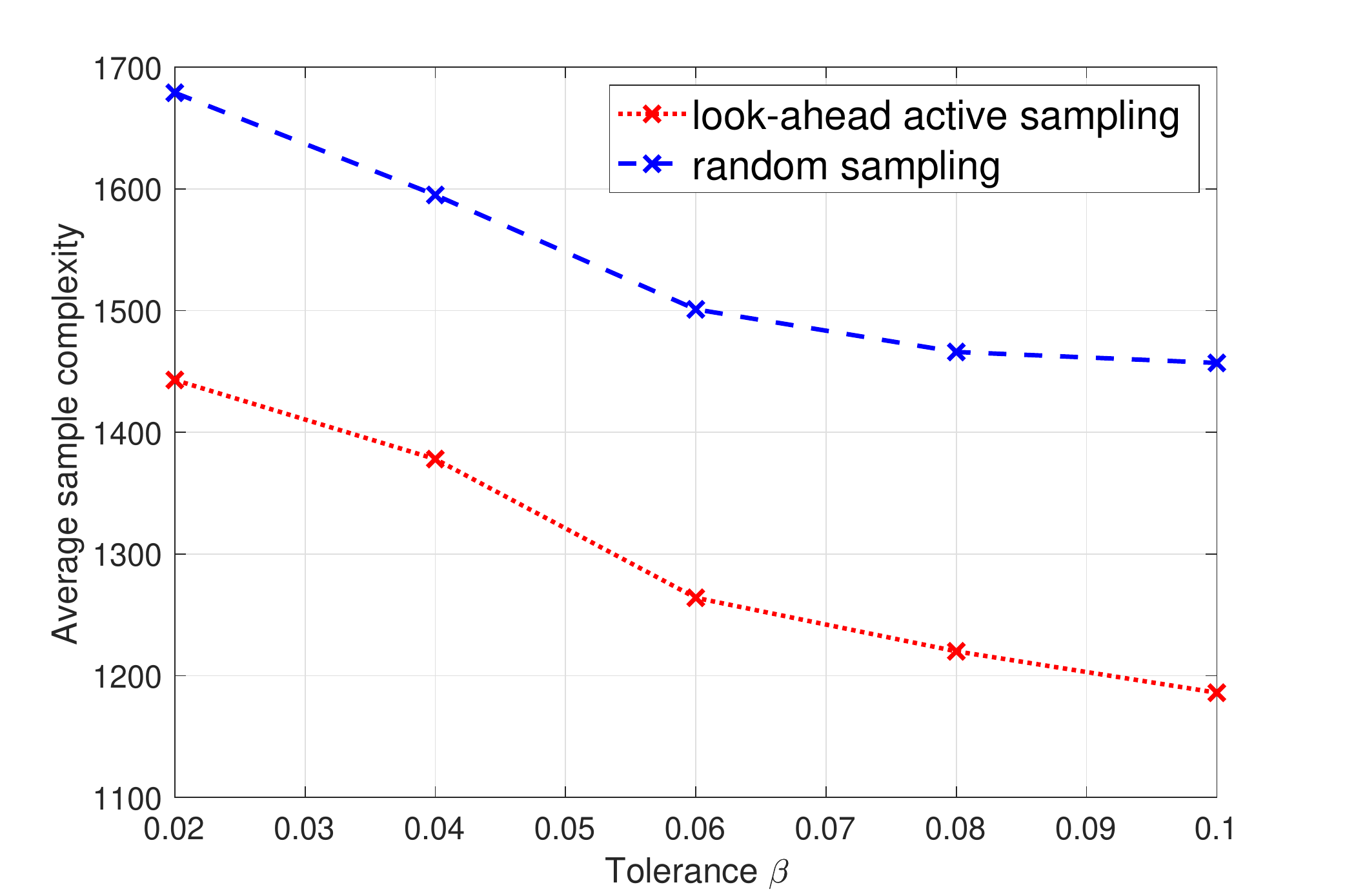} 
	\caption{Average sample complexity versus prescribed tolerance $\beta$.}
	\label{fig:RW} 
\end{figure}
	
	\section{Conclusions}
	In this paper, we have investigated the problem of active sequential estimation. Specifically, we have considered a setup in which the problem instance may have both shared as well as private parameters, and the goal is to estimate each of the parameters while meeting a prescribed level of confidence on their estimation qualities. Sequential decision rules are proposed, where the sampling rule involves maximizing the Fisher Information measure when the setting comprises shared parameters only (greedy sampling rule), while in the setting of both private and shared parameters, the rule maintains a trade-off between exploiting the most informative action and exploring the scarcely sampled actions in order to meet the prescribed tolerances on the estimation quality of every parameter. The proposed decision rules have been shown to be asymptotically optimal, and we have provided numerical experiments to evaluate their advantage compared to the existing strategies.
	
	\appendix
	
	\section{Useful Lemmas}
	In this section, we provide a few lemmas which we will be using throughout our analysis.
	\begin{lemma}[Lemma 4.1, \cite{Atia}]
		\label{lemma:atia4.1}
		When we have only a shared parameter (setting of Section \ref{section: single_param}), under assumptions A$_1$-A$_6$, for any $h>0$, there exists $t(h)\in \N$ such that for all $t>t(h)$:
		\begin{align}
			\P\Big\lbrace t\cdot \inf_{\hat\theta_t,\;\psi^t}\;{\sf C}(\hat\theta_t\;|\;\F_t)\;\geq\; \beta V_{\beta}(\theta)-h\Big\rbrace\;>\;1-h\;.
		\end{align}
	\end{lemma} 
	\begin{lemma}[Theorem 4.3, \cite{Atia}]
		\label{lemma:atia4.3}	
		Under assumptions A$_1$-A$_6$, the sequences of \textcolor{black}{estimates} generated by $\nu_t^{\sf MMSE}$, which is characterized in Section \ref{estimator}, and sampling rules $\psi^t$ specified in Section \ref{Control} satisfy 
		\begin{align}
			t\cdot {\sf C}(\nu_t^{\sf MMSE}\;|\;\F_t)\xrightarrow{t\rightarrow\infty}\beta V_{\beta}(\theta)\;.
		\end{align}
	\end{lemma}
	\begin{lemma}
		\label{lemma:con_multi}
		When we have both shared and private parameters (the setting in Section \ref{FullyUnknown}), under assumptions A$_1$-A$_6$, for any $h>0$, there exists $t^*(h)<+\infty$ such that for any $t>t^*(h)$:
		\begin{align}
			\P\Bigg\lbrace \inf\limits_{\Phi_t,\psi^t} \;t\bigg[\frac{1}{\beta}{\sf C}(\hat\theta_t|\F_t)+\sum\limits_{i\in[K]}\frac{1}{\beta_i}{\sf D}(\hat\alpha_{i,t}|\F_t)\bigg]&\;\geq\; W_{\bbeta}(\theta,\balpha) - h\Bigg\rbrace\nonumber\\&\;>\; 1 - (K+1)h\;.
		\end{align}
	\end{lemma}
	\begin{proof}
		Let us denote the number of times experiment $i\in[K]$ is selected up to time $t\in\N$ by 
		\begin{align}
			T_i(t)\;\triangleq\; \sum\limits_{s=1}^t \mathds{1}_{\lbrace \psi(s)=i\rbrace}\;.
		\end{align}
		From the Cram\'{e}r-Rao lower bound, we obtain the following inequalities.
		\begin{align}
			\label{eq:multi_crlb1}
			&\P\Big\lbrace {\sf C}(\hat\theta_t\;|\;\F_t)\;\geq\; \Big(\sum\limits_{i\in[K]}T_i(t)\mathscr{J}_i(\theta)\Big)^{-1}\Big\rbrace\;=\;1\;,\\
			&\P\Big\lbrace {\sf D}(\hat\alpha_{i,t}\;|\;\F_t)\;\geq\; \Big(T_i(t)\mathscr{J}_i(\alpha_i)\Big)^{-1}\Big\rbrace\;=\;1,\quad \forall i\in[K]\ .
			\label{eq:multi_crlb2}
		\end{align}
		Furthermore, let us denote the fraction of times any experiment $i\in[K]$ is selected up to time $t$ by
		\begin{align}
			p_i(t)\;\triangleq\; \frac{1}{t}\sum\limits_{s=1}^t \mathds{1}_{\lbrace \psi(s)=i\rbrace}\;.
		\end{align}
		Subsequently, from (\ref{eq:multi_crlb1}) and (\ref{eq:multi_crlb2}), we obtain that
		\begin{align}
			\label{eq:p_multi_crlb1}
			&\P\Big\lbrace t\cdot{\sf C}(\hat\theta_t\;|\;\F_t)\;\geq\; \Big(\sum\limits_{i\in[K]}p_i(t)\mathscr{J}_i(\theta)\Big)^{-1}\Big\rbrace\;=\;1\;,\\
			&\P\Big\lbrace t\cdot{\sf D}(\hat\alpha_{i,t}\;|\;\F_t)\;\geq\; \Big(p_i(t)\mathscr{J}_i(\alpha_i)\Big)^{-1}\Big\rbrace
			\label{eq:p_multi_crlb2}=\;1,\quad \forall i\in[K]\ .
		\end{align}
		Using the same argument as in~\cite[Lemma 4.1]{Atia},  any measurable control action will have a limiting distribution, or will be arbitrarily close to a limiting distribution such that for any arbitrarily small $h>0$, there exists $t(h)<+\infty$ such that for every $t>t(h)$, from (\ref{eq:p_multi_crlb1}) we obtain
		\begin{align}
			\label{eq:CRLB_multi_theta}
			\P\Big\lbrace t\cdot{\sf C}(\hat\theta_t\;|\;\F_t)\;\geq\; \Big(\sum\limits_{i\in[K]}p(i)\mathscr{J}_i(\theta)\Big)^{-1}-h\Big\rbrace\;>\;1-h\;,
		\end{align}
		where $p(i)$ denotes the limiting distribution of $p_i(t)$. Using similar arguments, for every $i\in[K]$ and any $h>0$, there exists $t_i^\prime(h)<+\infty$ such that for all $t\geq t_i^\prime(h)$:
		\begin{align}
			\label{eq:CRLB_multi_alpha1}
			\P\Big\lbrace t\cdot{\sf D}(\hat\alpha_{i,t}\;|\;\F_t)\;\geq\; \Big(p(i)\mathscr{J}_i(\alpha_i)\Big)^{-1}-h\Big\rbrace\;>\;1-h\;.
		\end{align}
		By defining $t^*(h)\triangleq\max \lbrace t(h), t_1^\prime(h),\cdots,t_K^\prime(h)\rbrace$, combining (\ref{eq:CRLB_multi_theta}) and (\ref{eq:CRLB_multi_alpha1}) and taking the infimum with respect to all control actions and \textcolor{black}{estimates}, we obtain that 
		\begin{align}
			\P\Bigg \lbrace\inf\limits_{\Phi_t,\psi^t} t\bigg[\frac{1}{\beta} {\sf C}(\hat\theta_t|\F_t) + \sum\limits_{i\in[K]}\frac{1}{\beta_i} {\sf D}(\hat\alpha_{i,t}|\F_t)\bigg]&\geq W_{\bbeta}(\theta,\balpha)-h\Bigg \rbrace\nonumber\\& > 1-(K+1)h\ ,
			\label{eq:CRLB_multi_3}
		\end{align}
		for all $t>t^*(h)$, where (\ref{eq:CRLB_multi_3}) is obtained from the fact that for any two events $A$ and $B$, $\P(A\cap B) \geq \P(A) + \P(B) - 1$. This completes our proof.
	\end{proof}
	\begin{lemma}
		\label{lemma:multi_ach}
		When we have both shared and private parameters (the setting in Section \ref{FullyUnknown}), under assumptions A$_1$-A$_6$, the sequences of \textcolor{black}{estimates} characterized in Section \ref{sec:estimator} and sampling rules specified in Section \ref{sec:ControlPolicy2} satisfy
		\begin{align}
			t\bigg[\frac{1}{\beta}{\sf C}(\nu_t^{\sf MMSE}|\F_t)+\sum\limits_{i\in[K]}\frac{1}{\beta_{i}}{\sf D}(\zeta_{i,t}^{\sf MMSE}|\F_t,\nu_t^{\sf MMSE})\bigg]\xrightarrow{t\rightarrow \infty}W_{\bbeta}(\theta,\balpha).
		\end{align}
	\end{lemma}
	\begin{proof}
		We denote the frequency of selecting experiment $i\in[K]$ under the sampling rule $\psi^t$ by 
		\begin{align}
			\label{eq:switching}
			p_i(t)\;\triangleq\;\frac{1}{t}\sum\limits_{s=1}^t\mathds{1}_{\lbrace \psi(s)\;=\;i\rbrace} \ .
		\end{align}
		Then, as $t\rightarrow\infty$, $p_i(t)$$\rightarrow \bar{q}_t(i)$ due to the law of large numbers. Furthermore, due to the strong consistency of ML estimates, ${\nu}_t^{\sf ML}$ converges to $\theta$, and ${\zeta}_{i,t}^{\sf ML}$ converges to $\alpha_{i}$ for all $i\in[K]$~\cite{MLE}.

		\textcolor{black}{Now, we will use a weak version of the Bernstein-Von-Mises theorem to establish the asymptotic normality of $\nu_t^{\sf MMSE}$. For this, let us consider a sequence of observations $\mathbf{Y}^t\triangleq \left\{Y_1,\cdots,Y_t\right\}$, which are not necessarily i.i.d. Let $\theta$ be an unknown parameter of the underlying pdfs of the observations $\mathbf{Y}^t$, and $\lambda(\cdot\;|\;\theta)$ denote the log-likelihood function corresponding to the observations. Next, we define the Fisher Information (FI) measure as
    \begin{align}
        \mathcal{J}(\theta)\triangleq -\frac{1}{t}\mathbb{E}\left[\frac{\partial^2}{\partial\theta^2}\lambda(\bY^t\;|\;\theta)\right]\ .
    \end{align}
    Furthermore, define $\nu_t^{\sf MMSE}$ as the MMSE \textcolor{black}{estimator}  of $\theta$, and define $z\triangleq \sqrt{t}(\nu_t^{\sf MMSE}-\theta)$. By a weak version of the Bernstein-Von-Mises theorem~\cite[Theorem 20.2]{dasgupta2008asymptotic}, we have
    \begin{align}
        \vartheta\left(z\;|\;\mathcal{F}_t\right) \xrightarrow{t\rightarrow\infty} \mathcal{N}(0,1/\mathcal{J}(\theta))\ ,
    \end{align}
    where $\vartheta(z\med\F_t)$ represents the posterior distribution of $z$. The key is to find $\mathcal{J}(\theta)$ in our case. We have,
    \begin{align}
        \mathcal{J}(\theta) &= -\frac{1}{t} \mathbb{E}\left[\frac{\partial^2}{\partial\theta^2}\lambda(\bY^t\;|\;\theta)\right]\\
        & = -\frac{1}{t}\mathbb{E}\left[\frac{\partial^2}{\partial\theta^2}\sum\limits_{i=1}^K \lambda_i(\bY^t_i\;|\;\theta)\right]\ ,
        \label{eq:bvm1}
    \end{align}
    where $\bY^t_i \triangleq \left\{Y_s : \psi(s) = i, s\in[t]\right\}$ is the sequence of observations from the experiment $i\in[K]$. Simplifying~(\ref{eq:bvm1}), we obtain
    \begin{align}
        \mathcal{J}(\theta) = \frac{1}{t}\sum\limits_{i=1}^K tp_i(t) \mathscr{J}_i(\theta)= \sum\limits_{i=1}^K p_i(t) \mathscr{J}_i(\theta)\ .
    \end{align}}
		Thus, we have
		\begin{align}
			\vartheta\big(z\lvert \F_t\big)\xrightarrow{t\rightarrow\infty}\mathcal{N}\bigg(0,\bigg (\sum\limits_{i\in[K]}p_i(t)\mathscr{J}_{i}(\theta) \bigg )^{-1}\bigg)\;.
		\end{align} 
		Furthermore, since $p(t)\triangleq[p_1(t),\cdots, p_K(t)]$ converges to the limiting distribution $q^*$ defined in (\ref{Control2_optimal}), we have
		\begin{align}
			\label{eq:convergence_of_nu}
			\vartheta\big(z\lvert \F_t\big)\xrightarrow{t\rightarrow\infty}\mathcal{N}\bigg(0,\bigg (\sum\limits_{i\in[K]}q^*(i)\mathscr{J}_{i}(\theta) \bigg )^{-1}\bigg)\;.
		\end{align}
		Thus, from (\ref{eq:convergence_of_nu}), we have $\nu_t^{\sf MMSE}\xrightarrow{p}\theta$. Using this fact in conjunction with the Bernstein-Von Mises theorem, for all $i\in[K]$, we have
		\begin{align}
			\label{eq:multi_ach1}
			\omega_i\big(u_i\med \nu_t^{\sf MMSE},\F_t\big)\xrightarrow{t\rightarrow\infty}\mathcal{N}\bigg(0,\bigg (q^*(i)\mathscr{J}_{i}(\alpha_{i}) \bigg )^{-1}\bigg)\;,
		\end{align}
		where we have defined $u_i\triangleq \sqrt{t}(\zeta_{i,t}^{\sf MMSE}-\alpha_i)$ and $\omega_i(u_i\med \nu_t^{\sf MMSE},\F_t)$ represents the posterior distribution of $u_i$.
		The limit of the sequence of MMSE is then calculated by establishing the finiteness of the first and second order moments of $\vartheta(z\med \F_t)$ and $\omega_i(u_i\med\nu_t^{\sf MMSE},\F_t)$. For $\vartheta(z\med\F_t)$, this can be done following the same approach of~\cite[Lemma 7.7]{Atia}. For $\omega_i(u_i\med\F_t)$, this can be done following the same approach of~\cite[ Lemma 3.1]{Yahav}. This yields
		\begin{align}\label{eq:ach2_slast}
			t\cdot{\sf C}(\nu_t^{\sf MMSE}\;|\;\F_t)&\xrightarrow{t\rightarrow\infty}\bigg(\sum\limits_{i\in[K]}q^*(i)\mathscr{J}_i(\theta)\bigg)^{-1}\;,\\
			\label{eq:ach2_last}
			\text{and}\quad t\cdot{\sf D}(\zeta_{i,t}^{\sf MMSE}\;|\;\F_t,\;\nu_t^{\sf MMSE})&\xrightarrow{t\rightarrow\infty}\Big(q^*(i)\mathscr{J}_i(\alpha_i)\Big)^{-1}\;\;\forall i\in[K].
		\end{align}
		Finally, combining (\ref{eq:ach2_slast}) and (\ref{eq:ach2_last}), we obtain
		\begin{align}
			t\bigg[\frac{1}{\beta}{\sf C}(\nu_t^{\sf MMSE}|\F_t)+\sum\limits_{i\in[K]}\frac{1}{\beta_{i}}{\sf D}(\zeta_{i,t}^{\sf MMSE}|\F_t,\nu_t^{\sf MMSE})\bigg]\xrightarrow{t\rightarrow \infty}W_{\bbeta}(\theta,\balpha).
		\end{align}
		This concludes our proof.
	\end{proof}
	
	\section{Proof of Theorem~\ref{theorem:SC_Lower_bound_single_param}}
	\label{appendix: sc_lower_bound_single_param}
	
	Consider a constant $L\in\N$. Using the constraint in $\mathcal{P}(\beta)$ defined in (\ref{eq:var}), we have
	\begin{align}
		1\;&=\; \P\Big\lbrace {\sf C}(\hat\theta_T\;|\;\F_T)\;\leq\; \beta \Big\rbrace\\
		&=\; \P\Big\lbrace {\sf C}(\hat\theta_T\;|\;\F_T)\;\leq\; \beta\;,\;T\;\leq\;L \Big\rbrace \; \nonumber\\&\quad+ \; \P\Big\lbrace {\sf C}(\hat\theta_T\;|\;\F_T)\;\leq\; \beta\;,\;T\;>\;L \Big\rbrace\ .
		\label{eq:T1andT2}
	\end{align}
	Define $T_L(i)\;\triangleq\; \sum\limits_{s=1}^L\mathds{1}_{\lbrace\psi(s)=i\rbrace}$ as the number of times that experiment $i\in[K]$ is chosen up to time $L$. Then, we have
	\begin{align}
		\label{eq:T1_1}
		{\sf C}(\hat\theta_T\;|\;\F_T)\;&\geq\;\Big(\sum\limits_{i\in[K]}T_T(i)\mathscr{I}_i(\theta)\Big)^{-1}\\
		&\geq\; \frac{1}{T}\Big(\sum\limits_{i\in[K]}\mathscr{I}_i(\theta)\Big)^{-1}\\
		&\geq\; \frac{1}{L}\Big(\sum\limits_{i\in[K]}\mathscr{I}_i(\theta)\Big)^{-1}\\
		&=\; C(L)\;,
		\label{eq:T1_2}	
	\end{align}
	where the first inequality is a result of the Cram\'{e}r-Rao lower bound~\cite{Kay97}, and $C(L)$ is a positive constant. Choosing $\beta\in(0, C(L))$ ensures that the first term in (\ref{eq:T1andT2}) becomes $\P\Big\lbrace {\sf C}(\hat\theta_T\;|\;\F_T)\;\leq\; \beta\;,\;T\;\leq\;L \Big\rbrace=0$. Using Lemma \ref{lemma:atia4.1}, there exists $t(h)$ such that for all $t>t(h)$, we have
	\begin{align}
		\label{eq:th1_1}
		\P\Big\lbrace t\cdot {\sf C}(\hat\theta_t\;|\;\F_t)\;\geq\; \beta V_{\beta}(\theta)-h\Big\rbrace\;>\;1-h\;.
	\end{align}
	Leveraging (\ref{eq:th1_1}), let us expand the the second term in (\ref{eq:T1andT2}). Choosing $L=t(h)$, we have
	\begin{align}
		\label{eq:proof_th1_1}
		&\P\Big\lbrace {\sf C}(\hat\theta_T\;|\;\F_T)\;\leq\; \beta\;,\;T\;>\;t(h) \Big\rbrace\nonumber\\
		&=\; \P\Big\lbrace T{\sf C}(\hat\theta_T|\F_T)\leq T\beta,T>t(h)| T{\sf C}(\hat\theta_T|\F_T)\geq\beta V_{\beta}(\theta)-h\Big\rbrace\nonumber\\
		&\qquad\times\;\P\Big\lbrace  T\cdot{\sf C}(\hat\theta_T\;|\;\F_T)\;\geq\;\beta V_{\beta}(\theta)-h\Big\rbrace\nonumber\\
		&+\P\Big\lbrace T{\sf C}(\hat\theta_T,\theta|\F_T)\leq T\beta,T>t(h)| T{\sf C}(\hat\theta_T|\F_T)<\beta V_{\beta}(\theta)-h\Big\rbrace\nonumber\\
		&\qquad\times\;\P\Big\lbrace  T\cdot{\sf C}(\hat\theta_T\;|\;\F_T)\;<\;\beta V_{\beta}(\theta)-h\Big\rbrace\\
		&\leq\; \P\Big\lbrace T{\sf C}(\hat\theta_T|\F_T)\leq T\beta, T{\sf C}(\hat\theta_T|\F_T)\geq\beta V_{\beta}(\theta)-h\Big\rbrace + h\\
		&\leq\;\P\Bigg\lbrace T\;\geq\;\frac{\beta V_{\beta}(\theta)-h}{\beta}\Bigg\rbrace \; + \;h\;,
		\label{eq:UB1}
	\end{align}
	where (\ref{eq:proof_th1_1}) follows from the law of total probability, and~(\ref{eq:UB1}) is a result of the stopping rule in~(\ref{eq:N}). Thus, combining (\ref{eq:T1andT2}) with (\ref{eq:UB1}), and due to the choice of $\beta<C(h)$, we obtain
	\begin{align}
		\P\Bigg\lbrace T\;\geq\;\frac{\beta V_{\beta}(\theta)-h}{\beta}\Bigg\rbrace\;\geq\; 1\;-\;h\;,
		\label{eq:UB2}
	\end{align}
	where we have defined 
	\begin{align}
		C(h)\;\triangleq\; \frac{1}{t(h)}\Bigg(\sum\limits_{i\in[K]}\mathscr{I}_i(\theta)\Bigg)^{-1}\ .
	\end{align}
	The result readily follows by applying Markov's inequality to (\ref{eq:UB2}).

	\section{Proof of Theorem ~\ref{theorem: SC_upper_bound_single_param}}
	\label{proof: SC_upper_bound_single_param}
	Using Lemma \ref{lemma:atia4.3}, we obtain that for any $h>0$ there exists $t(h)<+\infty$ such that for any $t>t(h)$, we have
	\begin{align}
		\P\big\lbrace t\cdot {\sf C}(\nu_t^{\sf MMSE}\;|\;\F_t)\;\leq\;\beta V_{\beta}(\theta)\;+\;h\big\rbrace\;=\;1\;.
	\end{align}
	Now, observe that if $T-1>t(h)$,
	\begin{align}
		1\;&=\;\P\big\lbrace (T-1)\cdot {\sf C}(\nu_{T-1}^{\sf MMSE}\;|\;\F_{T-1})\;\leq\;\beta V_{\beta}(\theta)\;+\;h\big\rbrace\\
		\label{eq:sp_SC1}
		&\leq\; \P\big\lbrace (T-1)\beta\;\leq\; \beta V_{\beta}(\theta)\;+\;h\big\rbrace\\
		&=\; \P\Big\lbrace T\;\leq\; \frac{\beta V_{\beta}(\theta)+h}{\beta}+1\Big\rbrace\;,
		\label{eq:UB_ach1}
	\end{align}
	where (\ref{eq:sp_SC1}) is a result of the stopping rule in~(\ref{eq:N}). Furthermore,
	\begin{align}
		\label{eq:UB_ach2}
		&\P\Big\lbrace T\;\leq\; \frac{\beta V_{\beta}(\theta)+h}{\beta}+1\Big\rbrace\nonumber\\
		&=\;\P\Big\lbrace T\;\leq\; \frac{\beta V_{\beta}(\theta)+h}{\beta}+1, T-1\;>\;t(h)\Big\rbrace\nonumber\\
		&+\; \P\Big\lbrace T\;\leq\; \frac{\beta V_{\beta}(\theta)+h}{\beta}+1\;\Big\lvert\; T-1\;\leq\;t(h)\Big\rbrace\P\big(T-1\;\leq\;t(h)\big)
	\end{align} 
	Now, $\P(T-1\;\leq\;t(h))=\P\lbrace{\sf C}(\nu_T^{\sf MMSE}\;|\;\F_T)\leq\beta,T\leq t(h)+1\rbrace$. Following the same steps as in (\ref{eq:T1_1})-(\ref{eq:T1_2}), we can show that there exists a constant $C^\prime(h)\triangleq (\sum\limits_{i\in[K]}\mathscr{I}_i(\theta))^{-1}/(t(h)+1)$, such that $\P(T-1\;\leq\;t(h))=0$ for any $\beta\in(0,C^\prime(h))$.
	Finally, combining (\ref{eq:UB_ach1}) and (\ref{eq:UB_ach2}), we obtain that for any $\beta\in (0,C^\prime(h))$, 
	\begin{align}
		\P\Big\lbrace T\;\leq\; \frac{\beta V_{\beta}(\theta)+h}{\beta}+1\Big\rbrace\;=\; 1 \;.
	\end{align}
	This completes the proof.
	
	\section{Proof of Theorem ~\ref{remark:ach_single_param}} \label{proof: avg_single_ach}

	Using Lemma~\ref{lemma:atia4.3}, for any $\epsilon>0$, there exists $T_\epsilon < +\infty$, such that for all $t\geq T_{\epsilon}$, 
	\begin{align}
		\label{eq: avg_ach_single1}
		t {\sf C}(\nu_t^{\sf MMSE}\med\F_t)\;\in\; [\beta V_{\beta}(\theta)-\epsilon, \beta V_{\beta}(\theta)+ \epsilon]\ .
	\end{align}
	Now, at the instant before stopping we have
	\begin{align}
		T-1\;&=\; (T-1)\mathds{1}_{\{T-1\leq T_{\epsilon}\}} + (T-1)\mathds{1}_{\{T-1>T_{\epsilon}\}}\\
		&\leq T_{\epsilon} + \frac{\beta V_{\beta}(\theta)+\epsilon}{\beta} + 1 \ ,
		\label{eq: avg_ach_single2}
	\end{align}
	where (\ref{eq: avg_ach_single2}) follows from the definition of the stopping rule and (\ref{eq: avg_ach_single1}). Furthermore, note that $\sup_{Y^t,\psi^t} T_{\epsilon} <~+~\infty$ owing to Lemma \ref{lemma:atia4.3}. Thus, taking average throughout (\ref{eq: avg_ach_single2}), dividing by $1/\beta$, and taking the limit of $\beta\rightarrow 0$, we obtain
	\begin{align}
		\lim\limits_{\beta\rightarrow 0} \frac{\E[T]}{V_{\beta}(\theta)} \;\leq\; 1\ .
	\end{align}
	
	\section{Proof of Theorem ~\ref{theorem:multi_sc_converse}}
	\label{proof:multi_sc_converse}
	Define the event
	\begin{align}
		\mathcal{S}_t\;\triangleq\; \Big\lbrace {\sf C}(\hat\theta_t\;|\;\F_t)\;\leq\;\beta,\; {\sf D}(\hat\alpha_{i,t}\;|\;\F_t,\;\hat\theta_t)\;\leq\;\beta_i\;\forall\;i\in[K]\Big\rbrace\ .
	\end{align} 
	From the constraint in (\ref{eq:var_multi}), for any constant $L<+\infty$, we have
	\begin{align}
		\label{eq:multi_T1_and_T2}
		1\;=\; \P\lbrace \mathcal{S}_T\rbrace\;=\;\P\lbrace \mathcal{S}_T,\;T\leq L\rbrace + \P\lbrace \mathcal{S}_T,\; T>L\rbrace\;.
	\end{align}
	By the Cram\'{e}r-Rao lower bound, if $T\leq L$,
	\begin{align}
		\label{eq:const1}
		&{\sf C}(\hat\theta_T\;|\;\F_T)\;\geq\; \frac{1}{L}\Big(\sum\limits_{i\in[K]}\mathscr{J}_i(\theta)\Big)^{-1}\ .
	\end{align}
	Similarly, for any $i\in[K]$ and $T\leq L$, we have
	\begin{align}
		\label{eq:const2}
		{\sf D}(\hat\alpha_{i,T}\;|\;\F_T,\;\hat\theta_{T})\;&\geq\; \frac{1}{L}\big(\mathscr{J}_i(\alpha_i\;|\;\hat\theta_T)\big)^{-1}\\
		&\geq\; \inf\limits_{\theta\in\Theta }\;\frac{1}{L}\big(\mathscr{J}_i(\alpha_i\;|\;\theta)\big)^{-1}\ ,
	\end{align}
	where we have defined:
	\begin{align}
		\mathscr{J}_i(\alpha_i\med\theta)\;\triangleq\; \E\Bigg [ \frac{\partial^2}{\partial\alpha_i^2}\lambda_i(x\med\theta,\alpha_i)\Bigg]\ .
	\end{align}
	Define the quantities
	\begin{align}
		\label{eq:const3}
		&C(L)\;\triangleq\; \frac{1}{L}\Bigg(\sum\limits_{i\in[K]}\mathscr{J}_i(\theta)\Bigg)^{-1}\ ,\nonumber\\ \text{and}\quad &D_i(L)\;\triangleq\; \inf\limits_{\theta\in\Theta} \frac{1}{L} \bigg(\mathscr{J}_i(\alpha_i\med \theta)\bigg)^{-1}\ .
	\end{align}
	Similarly to (\ref{eq:T1_1})-(\ref{eq:T1_2}), choosing $\beta\in(0,C(L))$ and $\beta_i\in(0,D_i(L))$ for every $i\in[K]$,
	\begin{align}
		\label{eq:multi_T1}
		\P\big(\mathcal{S}_T,\; T\leq L\big)\;=\;0\;.
	\end{align}
	Expanding the second term on the right hand side of (\ref{eq:multi_T1_and_T2}), we obtain
	\vspace{0.2in}
	\begin{align}
		&\P\lbrace \mathcal{S}_T,\; T>L\rbrace\nonumber\\
		&=\; \P\lbrace T\cdot{\sf C}(\hat\theta_T\;|\;\F_T)\;\leq\;T\beta,\nonumber\\&\qquad\qquad T\cdot{\sf D}(\hat\alpha_{i,T}\;|\;\F_T,\;\hat\theta_T)\;\leq\;T\beta_{i}\;\forall\; i\in[K],\; T>L\rbrace\\
		\label{eq:drop_theta}
		&\leq\; \P\lbrace T\cdot{\sf C}(\hat\theta_T\;|\;\F_T)\;\leq\;T\beta,\nonumber\\&\qquad\qquad T\cdot{\sf D}(\hat\alpha_{i,T}\;|\;\F_T)\;\leq\;T\beta_{i}\;\forall\; i\in[K],\; T>L\rbrace\\
		&\leq\; \P\Biggl\{\bigg[\frac{T{\sf C}(\hat\theta_T\;|\;\F_T)}{\beta}+\sum\limits_{i\in[K]}\frac{T{\sf D}(\hat\alpha_{i,T}\;|\;\F_T)}{\beta_{i}}\bigg]\;\leq\;(K+1)T,\nonumber\\&\qquad\qquad\qquad\qquad\qquad\qquad T>L\Biggr\}\ ,
		\label{eq:multi_conv1}
	\end{align}
	where (\ref{eq:drop_theta}) follows from the fact that since ${\sf D}(\hat\alpha_{i,T}\;|\;\F_T,\;\hat\theta_T)\;\leq\;\beta_{i}$ at stopping, we can take an expectation over ${\sf D}(\hat\alpha_{i,T}\;|\;\F_T,\;\theta)$ with respect to $\E_t$ and obtain ${\sf D}(\hat\alpha_{i,T}\;|\;\F_T)\;\leq\;\beta_{i}$, i.e.,
	$\{ {\sf D}(\hat\alpha_{i,T}\;|\;\F_T,\;\hat\theta_T)\;\leq\;\beta_{i}\}\subseteq \{{\sf D}(\hat\alpha_{i,T}\;|\;\F_T)\;\leq\;\beta_{i}\}$.
	Furthermore, for any $h>0$, define the event 
	\begin{align}
		\mathcal{E}_t\;\triangleq\; \Biggl\{t\bigg[\frac{{\sf C}(\hat\theta_t\;|\;\F_t)}{\beta}+\sum\limits_{i\in[K]}\frac{{\sf D}(\hat\alpha_{i,t}\;|\;\F_t)}{\beta_{i}}\bigg]\;\geq\;W_{\bbeta}(\theta,\balpha)-h\Biggr\}
	\end{align}
	Next, we will use Lemma \ref{lemma:con_multi} and set $L=t^*(h)$. Expanding (\ref{eq:multi_conv1}) we have
	\begin{align}
		&\P\Biggl\{\bigg[\frac{T{\sf C}(\hat\theta_T\;|\;\F_T)}{\beta}+\sum\limits_{i\in[K]}\frac{T{\sf D}(\hat\alpha_{i,T}\;|\;\F_T)}{\beta_{i}}\bigg]\;\leq\;(K+1)T,\nonumber\\&\qquad\qquad\qquad\qquad\qquad\qquad T>t^*(h)\Biggr\}\nonumber\\
		&=\;\P\Biggl\{\bigg[\frac{T{\sf C}(\hat\theta_T\;|\;\F_T)}{\beta}+\sum\limits_{i\in[K]}\frac{T{\sf D}(\hat\alpha_{i,T}\;|\;\F_T)}{\beta_{i}}\bigg]\;\leq\;(K+1)T,\nonumber\\&\qquad\qquad\qquad\qquad\qquad\qquad T>t^*(h),\;\mathcal{E}_T\Biggr\}\nonumber\\
		\label{eq:multi_conv2}
		&+\; \P\Biggl\{\bigg[\frac{T{\sf C}(\hat\theta_T\;|\;\F_T)}{\beta}+\sum\limits_{i\in[K]}\frac{T{\sf D}(\hat\alpha_{i,T}\;|\;\F_T)}{\beta_{i}}\bigg]\;\leq\;(K+1)T,\nonumber\\&\qquad\qquad\qquad\qquad\qquad\qquad T>t^*(h)\;\bigg\lvert\;\overline{\mathcal{E}_T}\Biggr\}\cdot\P\big(\overline{\mathcal{E}_T}\big)\\
		\label{eq:multi_conv3}
		&\leq\; \P\Bigl\{W_{\bbeta}(\theta,\balpha)-h\;\leq\;(K+1)T\Bigr\}\;+\;(K+1)h\;,
	\end{align}
	where (\ref{eq:multi_conv2}) is a result of applying the law of total probability, and (\ref{eq:multi_conv3}) is a result of Lemma \ref{lemma:con_multi}. Finally, (\ref{eq:multi_conv3}) can be rewritten using (\ref{eq:multi_T1_and_T2}) and (\ref{eq:multi_T1}), as follows.
	\begin{align}
		\P\biggl\{T\;\geq\;\frac{W_{\bbeta}(\theta,\balpha)-h}{K+1}\biggr\}\;>\;1-h\;.
	\end{align}
	Subsequently, applying the Markov's inequality yields
	\begin{align}
		\E[T]\;\geq\; \frac{W_{\bbeta}(\theta,\balpha)-h}{K+1}\cdot \Big(1-(K+1)h\Big)\;.
	\end{align}
	\section{Proof of Theorem ~\ref{theorem:multi_sc_ach}}
	\label{proof:multi_sc_ach}
	Using Lemma \ref{lemma:multi_ach}, we have the following two convergence properties.
	\begin{align}
		\label{eq:multi_ach_limits}
		t\cdot {\sf C}\Big(\nu_t^{\sf MMSE}\med\;\F_t\Big)&\xrightarrow{t\rightarrow\infty}\beta V(\theta)\;,\nonumber\\ \text{and}\quad
		t\cdot {\sf D}(\zeta_{i,t}^{\sf MMSE}\;|\;\F_t,\;\nu_t^{\sf MMSE})&\xrightarrow{t\rightarrow\infty}\beta_i V_i(\alpha_i)\;.
	\end{align}
	Thus, for any $h>0$, there exists $t(h)<+\infty$ such that for all $t>t(h)$,
	\begin{align}
		\label{eq:multi_sc_ach1}
		\P \biggl\{t\cdot
		{\sf C}(\nu_t^{\sf MMSE}\;|\;\F_t)\;\leq\; \beta V(\theta)+h\biggr\}\;=\;1\;,\\ \P \biggl\{t\cdot
		{\sf D}(\zeta_{i,t}^{\sf MMSE}\;|\;\F_t,\;\nu_t^{\sf MMSE})\;\leq\; \beta_i V_i(\alpha_i)+h\biggr\}\;=\;1\;.
		\label{eq:multi_sc_ach2}
	\end{align} 
	Define $(T-\tau_i)$ as the $i^{\rm th}$ time instant at which the guarantee on one of the parameters is met. Furthermore, define $S(i)\in[K]$ as the experiment to which the parameter belongs, in case that it is a private parameter. Thus, $\tau_{K+1}=0$. Furthermore, using (\ref{eq:multi_sc_ach1}) and (\ref{eq:multi_sc_ach2}), we obtain that for any $h>0$ and $T-\tau_i-1\geq t(h)$,  
	\begin{align}
		\label{eq: multi_ach4}
		&\P\Bigg\{T-\tau_i\leq \frac{V_{S(i)}(\alpha_{S(i)})+h}{\beta_{S(i)}} + 1\Bigg\}\;=\;1\ ,\\
		\text{or,}\quad&\P\Bigg\{T-\tau_i\leq \frac{\beta V(\theta)+h}{\beta} + 1\Bigg\}\;=\;1\ ,
		\label{eq: multi_ach5}
	\end{align} 
	depending on whether the parameter on which the tolerance guarantee is achieved is shared or private. Combining (\ref{eq: multi_ach4}) and (\ref{eq: multi_ach5}), we obtain:
	\begin{align}
		\label{eq: multi_ach6}
		\P\Bigg\{(K+1)T\leq \frac{\beta V(\theta)+h}{\beta} &+ \sum\limits_{i\in[K]}\frac{\beta_i V_i(\alpha_i)+h}{\beta_i}\nonumber\\&\quad + \sum\limits_{i=1}^K \tau_i + (K+1)\Bigg\}\;=\;1\ .
	\end{align}
	Furthermore, define $\beta_{\min}\triangleq\min\{\beta,\beta_1,\cdots, \beta_K\}$ and $\beta_{\max}\triangleq\max\{\beta,\beta_1,\cdots, \beta_K\}$. Note that using (\ref{eq:multi_sc_ach1}) and (\ref{eq:multi_sc_ach2}), along with the definition of $(t-\tau_i)$, for $t-\tau_i-1>t(h)$, we have
	\begin{align}
		\label{eq: multi_ach7}
		\P\Bigg\{T-\tau_i\leq V_{\max}(\theta,\balpha) +\frac{h}{\beta_{\min}}+1\Bigg\}\;=\; 1\ .
	\end{align}
	Similarly, (\ref{eq:multi_ach_limits}) combined with the definition of $(T-\tau_j)$ for any experiment $S(j)\in[K]$ and for $t-\tau_j>t(h)$, we have
	\begin{align}
		\label{eq: multi_ach8}
		&\P\Bigg\{T-\tau_j\geq \frac{\beta_{S(j)} V_{S(j)}(\alpha_{S(j)})-h}{\beta_{S(j)}}\Bigg\}\;=\;1\ ,\\
		\text{or,}\quad&\P\Bigg\{T-\tau_j\geq\frac{\beta V(\theta)-h}{\beta}\Bigg\}\;=\;1\ ,
		\label{eq: multi_ach9}
	\end{align}
	resulting in
	\begin{align}
		\label{eq: multi_ach10}
		\P\Bigg\{T-\tau_j\geq V_{\min}(\theta,\balpha) + \frac{h}{\beta_{\max}}\Bigg\} \;=\; 1 \ .
	\end{align}
	Combining (\ref{eq: multi_ach7}) and (\ref{eq: multi_ach10}), we obtain  that the following relationship holds with probability 1:
	{\small
	\begin{align}
		\label{eq: multi_ach11}
		\tau_j-\tau_i\leq V_{\max}(\theta,\balpha)&-V_{\min}(\theta,\balpha)\nonumber+\bigg(\frac{1}{\beta_{\min}} + \frac{1}{\beta_{\max}}\bigg)h+1\ .
	\end{align}
	}
	Combining (\ref{eq: multi_ach6}) and (\ref{eq: multi_ach11}), we obtain that with probability 1
{\small
	\begin{align}
	\nonumber 	 T & \leq 1+  \frac{W_{\bbeta}(\theta,\balpha)}{K+1} \\
		& + \frac{K}{K+1}\bigg(V_{\max}(\theta,\balpha)-V_{\min}(\theta,\balpha) +\bigg(\frac{1}{\beta_{\min}} + \frac{1}{\beta_{\max}}\bigg)h +1\bigg) \ .
	\end{align}
	}
	Finally, to ensure that $T-\tau_i-1>t(h)$ for every parameter, we follow the same steps as (\ref{eq:const1}), (\ref{eq:const2}) and (\ref{eq:const3}) in order to obtain constants $L(h)$ and $M_i(h)$ for every $i\in[K]$, such that $\beta\in(0,L(h))$ and $\beta_i\in(0,M_i(h))$ ensure the required conditions. The corresponding choices of $L(h)$ and $M_i(h)$ are given by:
	\begin{align}
		&L(h)\;\triangleq\; \frac{1}{t(h)}\Bigg(\sum\limits_{i\in[K]}\mathscr{J}_i(\theta)\Bigg)^{-1},\nonumber\\ \text{and}\quad &M_i(h)\;\triangleq\; \inf\limits_{\theta\in\Theta} \frac{1}{t(h)} \Bigg(\mathscr{J}_i(\alpha_i\med\theta)\Bigg)^{-1}\ .
	\end{align}
	\section{Proof of Theorem~\ref{theorem:multi_ach_avg}}
	\label{proof:multi_ach_avg}
	From Lemma \ref{lemma:multi_ach}, there exists $T_\epsilon<+\infty$ such that for all $t\geq T_{\epsilon}$, 
	\begin{align}
		\label{eq:multi_ach_avg1}
		t\Bigg\{\frac{{\sf C}(\nu_t^{\sf MMSE}\med\F_t)}{\beta} &+ \sum\limits_{i\in[K]} \frac{{\sf D}(\zeta_{i,t}^{\sf MMSE}\med\F_t)}{\beta_i}\Bigg\}\nonumber\\&\quad\in[W_{\bbeta}(\theta,\balpha)-\epsilon, W_{\bbeta}(\theta,\balpha)+\epsilon]\ .
	\end{align}
	Let $a\in[K]$ be the first experiment for which the estimation guarantee is achieved, and let $(T-\tau_a)$ denote the time instant at which this guarantee is achieved. Then, using (\ref{eq:multi_ach_avg1}), along with the stopping criterion in (\ref{eq:stop2}) we have
	\begin{align}
		T-\tau_a - 1\;&=\; (T-\tau_a-1)\mathds{1}_{\{T-\tau_a-1\leq T_\epsilon\}} \nonumber\\&\quad+ (T-\tau_a-1)\mathds{1}_{\{T-\tau_a-1>T_\epsilon\}}\\
		&\stackrel{(\ref{eq:multi_ach_avg1})}{\leq}\; T_\epsilon + \frac{W_{\bbeta}(\theta,\balpha)+\epsilon}{K+1} + 1\ .
	\end{align}
	Hence,
	\begin{align}
		\label{eq:multi_ach_avg2}
		T\;\leq\; T_\epsilon + \frac{W_{\bbeta}(\theta,\balpha)+\epsilon}{K+1} + \tau_a + 2 \ .
	\end{align}
	Furthermore, using (\ref{eq:ach2_slast}), there exists $T_{0}(\epsilon)<+\infty$ such that for all $t\geq T_{0}(\epsilon)$
	\begin{align}
		\label{eq:multi_ach_avg3}
		t{\sf C}(\nu_t^{\sf MMSE}\med\F_t)\in[\beta V(\theta) -\epsilon, \beta V(\theta) + \epsilon]\ ,
	\end{align}
	where $V(\theta)$ is defined in (\ref{eq:Vs}).
	Similarly, using (\ref{eq:ach2_last}) there exists $T_{i}(\epsilon)<+\infty$ for $i\in[K]$, such that for all $t\geq T_{i}(\epsilon
	)$
	\begin{align}
		\label{eq:multi_ach_avg4}
		t{\sf D}(\zeta_{i,t}^{\sf MMSE}\med\F_t)\in[\beta_i V_i(\alpha_i) -\epsilon, \beta_i V_i(\alpha_i) + \epsilon]\ ,
	\end{align}
	where $V_i(\alpha_i)$ is defined in (\ref{eq:Vs}). Let $T^*_\epsilon\triangleq \max\{T_0(\epsilon), \cdots, T_{K}(\epsilon)\}$. 
	Furthermore, let $b\in[K]$ be the experiment selected at the stopping time. Thus, $\tau_b=0$. Subsequently, for any $\epsilon>0$, (\ref{eq:multi_ach_avg4}) yields
	\begin{align}
		\label{eq:multi_ach_avg6}
		T\;&\leq\; T^*(\epsilon) + \frac{V_b(\alpha_b)+\epsilon}{\beta_b} + 2\\
		&\leq\; T^*(\epsilon) + V_{\max}(\theta,\balpha) +\frac{\epsilon}{\beta_b} + 2\ ,
	\end{align}
	where (\ref{eq:multi_ach_avg6}) is a result of the fact that at stopping, ${\sf D}(\zeta_{i,t}^{\sf MMSE}\med \nu_t^{\sf MMSE}, \F_t)\leq\beta_i$, which implies that ${\sf D}(\zeta_{i,t}^{\sf MMSE}\med\F_t)\leq\beta_i$, which is obtained by taking an expectation over ${\sf D}(\zeta_{i,t}^{\sf MMSE}\med\nu_t^{\sf MMSE},\F_t)$ with respect to the measure $\E_t^i$.
	Thus, noting that $\tau_a\leq T$, we have
	\begin{align}
		\label{eq:multi_ach_avg7}
		\tau_a \;\leq\; T^*_\epsilon + V_{\max}(\theta,\balpha) +\frac{\epsilon}{\beta_b} + 2\ .
	\end{align}
	Finally, combining (\ref{eq:multi_ach_avg2}) and (\ref{eq:multi_ach_avg7}), we obtain:
	\begin{align}
		\label{eq:th6_1}
		T\;\leq\; T_{\epsilon} + T^*_\epsilon + \frac{W_{\bbeta}(\theta,\balpha)}{K+1} + V_{\max}(\theta,\balpha) +\frac{\epsilon}{\beta_b} + 4\ .
	\end{align} 
	Note that $\sup_{Y_t,\psi^t} T_\epsilon < +\infty$, and $\sup_{Y_t,\psi^t} T^*_\epsilon < +\infty$ due to Lemma~\ref{lemma:multi_ach}. The proof is completed by taking the expectation on both sides of (\ref{eq:th6_1}), dividing by $W_{\bbeta}(\theta,\balpha)$, and taking the limit. 
	\section{Proof of Theorem \ref{remark:chernoff}}\label{R1}
	
	By contradiction, we prove that a purely Chernoff-based sampling strategy is sub-optimal. We begin by assuming that the Chernoff-based sampling strategy described in section \ref{chernoff_policy} combined with an almost surely finite stopping time $\tau^{\sf c}$ is asymptotically optimal. This implies that every constraint given in the problem (\ref{eq:var_multi}) is satisfied at stopping. First, let us define the following quantities that are instrumental for our argument:
	\begin{align}
		a^*\;\triangleq\;\argmax_{i\in[K]}\;\bigg\lbrace \mathscr{J}_i(\theta)+\mathscr{J}_i(\alpha_{i})\bigg\rbrace\;,
	\end{align}
	\begin{align}
		&\mathcal{U}\;\triangleq\; \big\lbrace \{\theta,\alpha_i\} : \theta\in\Theta, \alpha_i\in\mcA_i \;\forall i\in[K] \big\rbrace\;,\\
		&\mathcal{U}_i\;\triangleq\;\bigg\lbrace \theta\in\Theta,\alpha_i\in\mcA_i\; : \; \mathscr{J}_i(\theta) + \mathscr{J}_i(\alpha_i)\nonumber\\ &\qquad\qquad\qquad > \; \mathscr{J}_j(\theta) + \mathscr{J}_j(\alpha_i)\; \forall\;j\in[K]\setminus i\bigg\rbrace\;,
	\end{align}
	where $a^*$ represents the most informative experiment that maximizes the overall FI,	$\mathcal{U}$ is a set containing all possible pairs of the parameters $\{\theta,\alpha_i\}$ for every $i\in[K]$, and $\mathcal{U}_i$ represents a subset of $\mathcal{U}$ that maximizes the FI computed for the pair $\{\theta,\alpha_i\}$ for the experiment $i\in[K]$. It can be readily verified that $(\theta,\alpha_{a^*})\in\mathcal{U}_{a^*}$. Now, by the strong consistency of the ML estimate~\cite{MLE}, we have,
	\begin{align}
		{\nu}_t^{\sf ML}\rightarrow\theta,\quad\quad{\zeta}_{a^*,t}^{\sf ML}\rightarrow \alpha_{a^*} \quad\text{  a.s.  }
	\end{align}
	This implies that there exists a finite $m(\epsilon_1,\epsilon_2)\in\N$ such that, almost surely, $\lvert {\nu}_t^{\sf ML} - \theta\rvert\;<\;\epsilon_1$, and $\lvert{\zeta}_{a^*,t}^{\sf ML}-\alpha_{a^*}\rvert\;<\;\epsilon_2$ for every $t\;\geq\;m(\epsilon_1,\epsilon_2)$. We also note that the ground truth $(\theta,\alpha_{a^*})$ is an interior point of the set $\mathcal{U}_{a^*}$, since $\mathcal{U}_i$ is an open set by definition. Thus, if the FI measures $\mathscr{I}(\theta)$ and $\mathscr{J}(\alpha_i)$ for every $i\in[K]$ are continuous functions of the respective parameters $\theta$ and $\alpha_i$, there exists an $\epsilon$-ball around the point $(\theta,\alpha_{a^*})$, $\epsilon\;>\;0$, such that any set of parameters $\{\theta^\prime,\alpha_{a^*}^\prime\}$ within the $\epsilon$-ball also maximizes the FI. Selecting $m(\epsilon_1,\epsilon_2)$ to be sufficiently large, such that $\epsilon\;<\;\min(\epsilon_1,\epsilon_2)$, we have that $({\nu}_t^{\sf ML},{\zeta}_{a^*,t}^{\sf ML})\in\mathcal{U}_{a^*}$ for every $t\geq m(\epsilon_1,\epsilon_2)$. Next, we will show that for a range of the confidence interval $\beta$, the optimal stopping rule uses at least $m(\epsilon_1,\epsilon_2)$ samples with a high probability. Combining the fact that $\bar{p}_t\rightarrow p^*$, where $\bar p_t$ represents the probability mass function over the experiments $i\in[K]$ due to the Chernoff-based sampling strategy, and $p^*$ denotes the distribution which selects the experiment $a^*$ with probability $1$, and the Bernstein-Von Mises Theorem (\cite{dasgupta2008asymptotic}), we obtain 
	\begin{align}
		\label{eq:chernoff_convergence}
		t\cdot{\sf C}(\nu_t^{\sf MMSE}\;|\;\F_t^{\sf c})\xrightarrow{t\rightarrow\infty}1/\mathscr{J}_{a^*}(\theta)\;,
	\end{align}
	where we have defined $\F_t^{\sf c}\triangleq\{Y^t,\psi_{\sf c}^t\}$ and $\psi_{\sf c}^t\triangleq\{\psi^{\sf c}(1),\cdots,\psi^{\sf c}(t)\}$. Let $\tau^{\sf c}$ denote the optimal stopping rule that minimizes (\ref{eq:var_multi}) under $\psi_{\sf c}^t$. Using (\ref{eq:chernoff_convergence}), there exists $\tau_\epsilon<+\infty$, such that for all $t\geq\tau_\epsilon$, we have
	\begin{align}
		1\;&=\;\P\Big\lbrace t\cdot{\sf C}(\nu_t^{\sf MMSE}\;|\;\F_{t}^{\sf c})\;\geq\; 1/\mathscr{J}_{a^*}(\theta)-\epsilon\Big\rbrace\ .
	\end{align}
	Furthermore, choosing $\tau^* = \max\{\tau_\epsilon, m(\epsilon_1,\epsilon_2)\}$, we have that for all $t\geq\tau^*$,
	\begin{align}
		1\;&=\;\P\Big\lbrace \tau^{\sf c}\cdot{\sf C}(\nu_{\tau^{\sf c}}^{\sf MMSE}\;|\;\F_{\tau^{\sf c}}^{\sf c})\;\geq\; 1/\mathscr{J}_{a^*}(\theta)-\epsilon,\tau^{\sf c}\;\leq\;\tau^*\Big\rbrace\nonumber\\
		\nonumber &\quad +\;\P\Big\lbrace \tau^{\sf c}\cdot{\sf C}(\nu_{\tau^{\sf c}}^{\sf MMSE}\;|\;\F_{\tau^{\sf c}}^{\sf c})\;\geq\; 1/\mathscr{J}_{a^*}(\theta)-\epsilon,\tau^{\sf c}\;>\;\tau^*\Big\rbrace\\
		&\leq\; \P\Big\lbrace \tau^{\sf c}\beta\;\geq\; 1/\mathscr{J}_{a^*}(\theta)-\epsilon,\tau^{\sf c}\;\leq\;\tau^*\Big\rbrace\nonumber\\
				\label{eq:chern1}
 &\quad+\;\P\Big\lbrace \tau^{\sf c}\cdot{\sf C}(\nu_{\tau^{\sf c}}^{\sf MMSE}\;|\;\F_{\tau^{\sf c}}^{\sf c})\;\geq\; 1/\mathscr{J}_{a^*}(\theta)-\epsilon,\tau^{\sf c}\;>\;\tau^*\Big\rbrace\;,
	\end{align}
	where (\ref{eq:chern1}) holds since at the stopping time, ${\sf C}(\nu_{\tau^{\sf c}}^{\sf MMSE}\med\F_{\tau^{\sf c}}^{\sf c})\leq\beta$. Selecting $\beta<\frac{1}{\tau^*}\Big(\frac{1}{\mathscr{J}(\theta)}-\epsilon\Big)$, we observe that $\P\Big\lbrace \tau^{\sf c}\beta\;\geq\; 1/\mathscr{J}_{a^*}(\theta)-\epsilon,\tau^{\sf c}\;\leq\;\tau^*\Big\rbrace=0$. Thus, we have
	\begin{align}
		\P\Big\lbrace \tau^{\sf c}\;>\;\tau^*\Big\rbrace\;=\;1\;.
	\end{align}
	Finally, note that for all $t>\tau^*$ and for each experiment $i\in[K]\setminus a^*$, ${\sf D}(\zeta_{i,{\tau^{\sf c}}}^{\sf MMSE}|\F_{\tau^{\sf c}}^{\sf c},\;\nu_{\tau^{\sf c}}^{\sf MMSE})=$ ${\sf D}(\zeta_{i,{\tau^*}}^{\sf MMSE}|\F_{\tau^*}^{\sf c},\;\nu_{\tau^*}^{\sf MMSE})$, since for all $t\geq\tau^*$, the Chernoff-based control action selects experiment $a^*$. Let $\gamma\triangleq\min\limits_{i\in[K]\setminus a^*}{\sf D}(\zeta_{i,{\tau^{\sf c}}}^{\sf MMSE}|\F_{\tau^{\sf c}}^{\sf c},\nu_{\tau^{\sf c}}^{\sf MMSE})$, $\gamma>0$. Thus, if we choose $\beta_i<\gamma$ for any $i\in[K]\setminus a^*$, the sampling strategy forces the estimation cost ${\sf C}(\zeta_{i,{\tau^c}}^{\sf MMSE}|\F_{\tau^{\sf c}}^{\sf c},\;\nu_{\tau^{\sf c}}^{\sf MMSE})>\beta_i$, hence, failing to satisfy the constraints in~(\ref{eq:var_multi}), which is a contradiction to our initial assumption.

	\bibliographystyle{IEEEtran}
	\bibliography{CSERef3.bib}
	\vspace{0.2in}

\end{document}